%% file: iclr2026_conference.tex
\documentclass{article}
\PassOptionsToPackage{numbers, compress}{natbib}
\usepackage{iclr2026_conference,times}
\iclrfinalcopy

\PassOptionsToPackage{table}{xcolor}
\usepackage[dvipsnames]{xcolor}

\input{math_commands}

\usepackage{natbib} % optional if class already loads natbib
\setcitestyle{numbers} % ensure numeric mode
\usepackage{microtype}
\usepackage{graphicx}
\usepackage{subfig}
\usepackage{booktabs} % for professional tables
\usepackage{csquotes}
\usepackage{tcolorbox}
\usepackage{enumitem}  % Gives more control over list formatting
\usepackage{colortbl}
\usepackage{booktabs}
\usepackage{amssymb}
\usepackage{amsthm}
\usepackage{tcolorbox}
\usepackage{pifont}

\usepackage{algorithm}        % floating algorithm environment
\usepackage{algorithmic}      % defines \STATE, \FOR, \IF, etc.

\usepackage{booktabs,tabularx}
\usepackage{booktabs,wrapfig}
\usepackage{hyperref}
\usepackage{multicol}
\usepackage{multirow}
\usepackage{amsmath}
\usepackage{amssymb}
\usepackage{mathtools}
\usepackage{amsthm}
\usepackage{times}
\usepackage{etoc}
\definecolor{darkpurple}{RGB}{48, 25, 52}
\definecolor{darkbrick}{RGB}{128, 0, 0}  
\definecolor{darkforest}{RGB}{0, 59, 0}
\definecolor{darknavy}{RGB}{0, 0, 89}

\hypersetup{
   colorlinks,
   citecolor=[rgb]{0.00, 0.35, 0.90},    % Darker blue, brighter tone
   linkcolor=[rgb]{0.01, 0.62, 0.45},    % Brighter green, darker base
   urlcolor=[rgb]{0.95, 0.35, 0.85},     % Brighter purple, darker base
}

\usepackage[capitalize,noabbrev]{cleveref}

\newtheorem{theorem}{Theorem}
\newtheorem*{theorem*}{Theorem}
\newtheorem*{lemma*}{Lemma}
\newtheorem{lemma}[theorem]{Lemma}

\newtheorem{definition*}{Definition}

\newcommand{\dd}{\mathrm{d}}
\usepackage[textsize=tiny]{todonotes}
\newcommand{\ours}{LoRA-SB }
\newcommand{\xs}{LoRA-XS }
\newcommand{\lora}{LoRA }

\title{Initialization using Update  Approximation is a \textit{Silver Bullet} for Extremely Efficient Low-Rank Fine-Tuning}

\author{
  \textbf{Kaustubh Ponkshe}\textsuperscript{*1}, 
  \textbf{Raghav Singhal}\textsuperscript{*1}, 
  \textbf{Eduard Gorbunov}\textsuperscript{1}, 
  \textbf{Alexey Tumanov}\textsuperscript{2},\\
  \textbf{Samuel Horvath}\textsuperscript{1}, 
  \textbf{Praneeth Vepakomma}\textsuperscript{1,3} \\
  \textsuperscript{1} Mohamed bin Zayed University of Artificial Intelligence, UAE \\
  \textsuperscript{2} Georgia Institute of Technology, USA
  \textsuperscript{3} Massachusetts Institute of Technology, USA
}

\footnotetext{* denotes equal contribution. Author order decided randomly.}
\begin{document}
 \maketitle

\begin{abstract}
Low-rank adapters have become standard for efficiently fine-tuning large language models, but they often fall short of achieving the performance of full fine-tuning. We propose a method, \textbf{LoRA} \textbf{S}ilver \textbf{B}ullet or \textbf{LoRA-SB}, that approximates full fine-tuning within low-rank subspaces using a carefully designed initialization strategy. We theoretically demonstrate that the architecture of LoRA-XS, which inserts a learnable \( r \times r \) matrix between \( B \) and \( A \) while keeping other matrices fixed, provides the precise conditions needed for this approximation. We leverage its constrained update space to achieve optimal scaling for high-rank gradient updates while removing the need for scaling factor tuning. We prove that our initialization offers an optimal low-rank approximation of the initial gradient and preserves update directions throughout training. Extensive experiments across mathematical reasoning, commonsense reasoning, and language understanding tasks demonstrate that our approach exceeds the performance of LoRA (and baselines) while using \textbf{27-90} times fewer learnable parameters, and comprehensively outperforms LoRA-XS. Our findings establish that it is possible to simulate full fine-tuning in low-rank subspaces, and achieve significant parameter efficiency gains without sacrificing performance.
Our code is publicly available at: \url{https://github.com/CERT-Lab/lora-sb}.
% Anonymous code is available at: \url{https://anonymous.4open.science/r/lora-sb-anonymous-5BEE}.
\end{abstract}

\input{introduction}

\input{related_work}
\input{method}

\input{experiments}

\input{analysis}

\input{conclusion}
%\input{reproducibility_statement}
%\input{impact_statement}
\input{ack}
\bibliography{example_paper}
\bibliographystyle{plain}

%%%%%%%%%%%%%%%%%%%%%%%%%%%%%%%%%%%%%%%%%%%%%%%%%%%%%%%%%%%%%%%%%%%%%%%%%%%%%%%
%%%%%%%%%%%%%%%%%%%%%%%%%%%%%%%%%%%%%%%%%%%%%%%%%%%%%%%%%%%%%%%%%%%%%%%%%%%%%%%
% APPENDIX
%%%%%%%%%%%%%%%%%%%%%%%%%%%%%%%%%%%%%%%%%%%%%%%%%%%%%%%%%%%%%%%%%%%%%%%%%%%%%%%
%%%%%%%%%%%%%%%%%%%%%%%%%%%%%%%%%%%%%%%%%%%%%%%%%%%%%%%%%%%%%%%%%%%%%%%%%%%%%%%
\clearpage
\appendix

% make an unnumbered “Appendix” heading as the parent
\part*{Appendix}
\addcontentsline{toc}{part}{Appendix} % if you want it in the main ToC

% now the local ToC of all \section{…} below
\etocsettocdepth{subsection}
\localtableofcontents
% ————————————————

\input{appendix}

\end{document}

%% file: math_commands.tex
\usepackage{amsmath,amsfonts,bm}

\def\eqref#1{equation~\ref{#1}}
\def\1{\bm{1}}

\DeclareMathAlphabet{\mathsfit}{\encodingdefault}{\sfdefault}{m}{sl}
\SetMathAlphabet{\mathsfit}{bold}{\encodingdefault}{\sfdefault}{bx}{n}

\DeclareMathOperator*{\argmin}{arg\,min}

%% file: introduction.tex
\section{Introduction}
Pre-trained language models have become central to natural language processing, achieving state-of-the-art performance across diverse tasks \cite{Radford2021LearningTV,Kirillov2023SegmentA, achiam2023gpt}. While these models excel at general-purpose capabilities \cite{bubeck2023sparks, Hao_Song_Dong_Huang_Chi_Wang_Ma_Wei_2022}, adapting them to specific downstream tasks often requires fine-tuning (FT). 
%Although in-context learning \cite{Brown2020LanguageMA,Radford2019LanguageMA} has gained popularity  for its simplicity, it falls short in both performance and efficiency compared to FT \cite{Liu_Tam_Muqeeth_Mohta_Huang_Bansal_Raffel_2022}. 
At the same time, full FT, while highly effective, is computationally expensive and impractical at scale.
% At the same time, full FT, while highly effective, is computationally expensive and impractical at scale, highlighting the need for more efficient adaptation techniques.

% Parameter-efficient fine-tuning (PEFT) has emerged as a crucial approach for adapting large language models (LLMs) while addressing computational constraints. While full FT typically achieves optimal performance,it is computationally intensive and often impractical for most applications. Low-rank decomposition methods, beginning with LoRA \cite{lora}, have shown particular promise by significantly reducing learnable parameters through learning low-rank updates. This has sparked numerous advances in low-rank methods that either enhance performance through better optimization techniques and initialization strategies, or improve parameter efficiency through structured matrices and adaptive rank selection \citep{Zhang_Chen_Bukharin_Karampatziakis_He_Cheng_Chen_Zhao_2023, LoRA-Pro, lora_ga}. However, these methods still face fundamental trade-offs, they must either maintain a relatively large number of parameters to match full FT performance, or accept a performance degradation when pursuing extreme parameter efficiency \cite{lora, Ding_Qin_Yang_Wei_Yang_Su_Hu_Chen_Chan_Chen_etal._2023,LoRA-Pro}. This raises an important question: \textbf{can we design low-rank methods that maintain full FT-competitive performance while drastically reducing the parameter count beyond current approaches?}
Parameter-efficient fine-tuning (PEFT) has become vital for adapting large language models (LLMs) under computational constraints. 
% While full fine-tuning (FT) achieves optimal performance, it is often computationally prohibitive. 
Low-rank methods like LoRA \cite{lora} address this by reducing learnable parameters via low-rank updates, sparking advancements in optimization, initialization, structured matrices, and adaptive rank selection \citep{Zhang_Chen_Bukharin_Karampatziakis_He_Cheng_Chen_Zhao_2023, LoRA-Pro, lora_ga}. However, these methods face trade-offs: either retain many parameters to match full FT or sacrifice performance for extreme efficiency \cite{lora, Ding_Qin_Yang_Wei_Yang_Su_Hu_Chen_Chan_Chen_etal._2023,LoRA-Pro}. This raises a critical question: Can we design low-rank methods that achieve full FT-level performance while drastically reducing parameter counts?

\begin{figure*}[!htbp]
    \centering
    \includegraphics[width=1\linewidth]{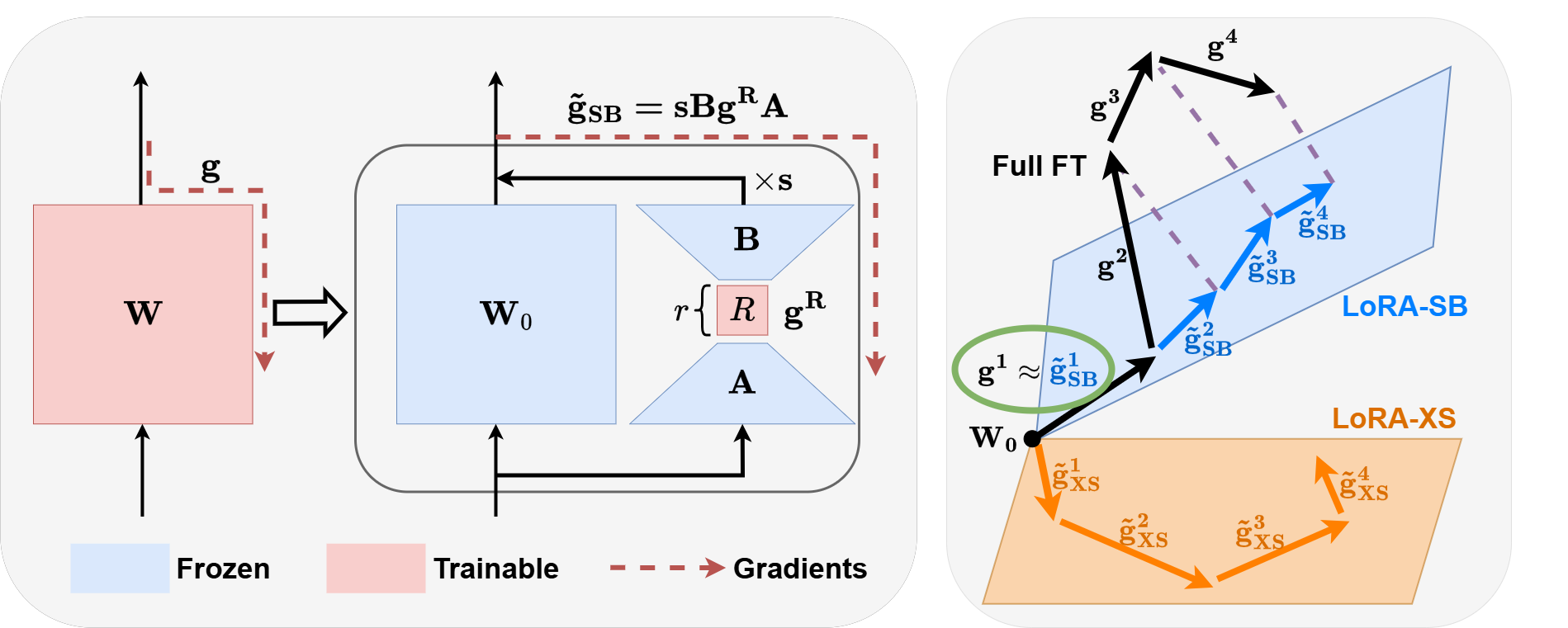}
    \caption
 %    {
 %    \textbf{LoRA-SB Illustration.} LoRA-XS \citep{LoRA-XS} reduces parameter count compared to LoRA \citep{lora} by inserting a learnable \( r \times r \) matrix \( R \) between \( B \) and \( A \), while keeping other matrices fixed, leading to \( W = W_0 + sBRA \). Our approach leverages the same architecture. We find that  updating \( R \) using its gradients \( g^R \)  is equivalent to updating the full-finetuning matrix \( W \) with a gradient \( sBg^RA \)  . We initialize \( B \), \( R \), and \( A \) from the first step of full FT, ensuring an optimal low-rank approximation of the initial gradient, \( g_1 \) (indicated in green), and preserving optimal update directions, \( \tilde{g}_{SB} \), throughout training (in purple). In contrast, LoRA-XS initializes from the pre-trained weights \( W \), which is suboptimal, and its gradients fail to approximate full FT updates accurately.
 % }
  {
    \textbf{LoRA-SB.} 
    % LoRA-XS \citep{LoRA-XS} reduces parameter count compared to LoRA \citep{lora} by inserting a learnable \(r \times r \) matrix \( R\) between \(B \) and \(A\), while keeping other matrices fixed, leading to \( W = W_0 + sBRA \). 
    % Our approach leverages the same architecture. We find that updating \( R \) using its gradients \( g^R \) is equivalent to updating the full-finetuning matrix \(W\) with an equivalent gradient \(\tilde{g}_{SB}=sBg^RA\). 
    % Through LoRA-SB, we ensure that the matrices \(B\), \(R\), and \(A\) are initialized such that the equivalent gradient \(\tilde{g}_{SB}\) optimally approximates the full FT gradient \(g\) in low rank subspaces \textbf{at each step}. 
    % In essence, we simulate the \textbf{entire full FT process} in low rank subspaces, by \textbf{utilizing only the first gradient \(g_1\) }(indicated in green) from full FT. 
    LoRA-XS \citep{LoRA-XS} reduces parameters compared to LoRA \citep{lora} by inserting a learnable \(r \times r \) matrix \( R\) between \(B \) and \(A\), while keeping other matrices fixed, leading to \( W = W_0 + sBRA \). 
    Our method, LoRA-SB, uses the same architecture.
    We find that updating \(R\) using its gradients \( g^R\) is equivalent to updating the full FT matrix \(W\) with an equivalent gradient \(\tilde{g}_{SB}=sBg^RA\).  
    We initialize \(B\), \(R\), and \(A\) such that the equivalent gradient \(\tilde{g}_{SB}\) provably best approximates the full FT gradient \(g\) in low rank subspaces \textbf{at each step}.
    In essence, we simulate the \textbf{entire full FT process} optimally within low-rank subspaces by \textbf{utilizing only the first full FT gradient \(g_1\)}.
 }
    \label{fig:lora-sb}
\end{figure*}

% Decomposotion methods all work on the fundamental assumption that only a low rank update is required to be learnt for FT. Some works go one step further and hypothesize that LoRA can learn any low rank full finetuning gradient. However the gradient which LoRA or even LoRA XS actually computes to update the learnable adapters, do not actually obey this property. In case of LoRA, gradient needs to be optimized at each step to male it approximate the full fien-tuning gradient.  Furthermore, another area which is known to be critical is Low rank adapters is their initialization. Pissa LoRA nad LoRA GA both highligh this problem and address it using some form of low rank approximation. 

% We identify similar shortcomings for LoRA XS and formalize why the problems are even more significant for it. We then propose solutions based on LoRA methods to overcome these limitations. However, after doing this we observe a very critical feature unique to LoRA XS: By setting appropriate assignments for $A$ and $B$ matrices, all the problems of LoRA XS get solved in one shot. Moreover, the gradients are scaled perfectly to approximate the high rank full FT gradeits. This helps us prove that the method is hyperparameter $\alpha$ independent. Our contributions can be given as follows:
Low-rank decomposition methods operate on a fundamental premise: FT requires learning only a low-rank update to the pre-trained weights. 
%Some theoretical work extends this hypothesis, suggesting that methods like LoRA can learn any low-rank approximation of the full FT gradient. 
However, the \textbf{gradients} computed by these methods do not inherently possess this property. 
For instance, LoRA's gradients need explicit optimization at each step to better approximate the full FT gradient \citep{LoRA-Pro}. Additionally, initialization has emerged as a critical factor in low-rank adaptation, as highlighted by recent works like PiSSA-LoRA \cite{pissa} and LoRA-GA \cite{lora_ga}.

We analyze these limitations in the context of the architecture of LoRA-XS \cite{LoRA-XS}, which inserts a learnable \( r \times r \) matrix between \( B \) and \( A \) while keeping other matrices fixed, and demonstrate that these challenges are even more pronounced. While exploring solutions inspired by LoRA-based methods, we discover a remarkable property unique to LoRA-XS: through careful initialization of $A$ and $B$, we can simulate the full FT optimization in low rank subspaces through \textbf{entire training}, as shown in Figure \ref{fig:lora-sb}. 
Our initialization provides optimal scaling for approximating high-rank full FT gradients and eliminates need for tuning the hyperparameter $\alpha$. 
\textbf{The peak memory usage of LoRA-SB never exceeds that of LoRA} or other baselines, and its training-time overhead relative to LoRA is negligible ($\approx 1.1\%-1.3\%$).
Our key contributions are:
\begin{itemize}[leftmargin=*,itemsep=2pt]
    \item We formalize the limitations of LoRA-XS, showing how its constrained update space leads to suboptimal gradient approximation, initialization sensitivity, and scaling dependence.
    
    \item We propose an initialization strategy derived from using the first step of full FT, which provides an optimal approximation of the initial gradient and preserves update directions throughout.
    
    \item We prove our initialization makes gradient optimization scaling-independent and guarantees convergence by maintaining orthonormal bases, eliminating need for tuning the scaling factor $\alpha$.
    
    \item Through extensive experiments on $4$ models across $16$ datasets covering mathematical reasoning, commonsense reasoning, and language understanding, we demonstrate that LoRA-SB surpasses LoRA while using \textbf{27-90x} less learnable parameters, and comprehensively outperforms LoRA-XS.
\end{itemize}

%% file: method.tex
\section{Methodology} \label{sec:Methodologies}

\subsection{Preliminaries} \label{subsec:Prelims}
In standard FT, a pre-trained weight matrix $W \in \mathbb{R}^{m \times n}$ is updated using the update matrix $\Delta W$ as:
\begin{equation}
W = W_0 + \Delta W
\end{equation}
where $W_0$ is the pre-trained weight. This requires updating $mn$ parameters per layer.
LoRA posits that updates lie in a low-dimensional subspace, parameterizing $\Delta W$ as:
\begin{equation}
W = W_0 + sBA  
\end{equation}
where $B \in \mathbb{R}^{m \times r}$ and $A \in \mathbb{R}^{r \times n}$ are trainable low-rank matrices with rank $r \ll \min(m,n)$, and $s$ is a scaling factor ($\alpha/r$) to stabilize training. This reduces the number of parameters from $mn$ to $r(m+n)$.
LoRA-XS efficiently parameterizes as:
\begin{equation}
W = W_0 + sBRA
\end{equation}
% where $B \in \mathbb{R}^{m \times r}$ and $A \in \mathbb{R}^{r \times n}$ are fixed matrices, and only $R \in \mathbb{R}^{r \times r}$ is trainable, reducing parameters to $r^2$.
% We denote the full FT gradient as $g = \frac{\partial L}{\partial W}$, and the LoRA-XS gradient as $g^R_{\text{LoRA-XS}} = \frac{\partial L}{\partial R}$.
where $B$ and $A$ are fixed, and only $R \in \mathbb{R}^{r \times r}$ is trainable, reducing the number of parameters to $r^2$.
We denote the full FT gradient: $g = \frac{\partial L}{\partial W}$; LoRA-XS gradient: $g^R_{\text{LoRA-XS}} = \frac{\partial L}{\partial R}$; $L$ is the loss function.

\subsection{Motivation} \label{subsec:Motiv}
% \subsection{Problems with Existing Approaches}

% \xs \cite{LoRA-XS}, despite having significantly fewer learnable parameters compared to LoRA, exhibits suboptimal performance. 

\xs \cite{LoRA-XS} has significantly fewer learnable parameters than LoRA but performs suboptimally. LoRA-XS's architecture causes constraints on the type of updates it can learn. 
The subspace of learned updates is characterized in Lemma \ref{lemma:subspace}. 
This implies that while $\Delta W$ is constrained to be rank $\leq r$, it also needs to have column and row spaces defined by those of $B$ and $A$, respectively. In contrast, \lora can learn any update $\Delta W$ as long as $\text{rank}(\Delta W) \leq r$. 
Thus, the low expressivity of \xs as compared to \lora can account for the performance drop.

\begin{tcolorbox}[colback=cyan!10,colframe=black]
\begin{lemma}\label{lemma:Subspace}
\label{lemma:subspace}
Let $\Delta W$ be an update learned with LoRA-XS. Then, the set of all possible $\Delta W$, say $ \mathcal{  W}_{LoRA-XS}$, is given as:
$$
 \mathcal{  W}_{LoRA-XS} =\{M \in \mathbb{R}^{m \times n}|
\text{Col}(M) \subseteq \text{Col}(B) \wedge \text{Row}(M) \subseteq  \text{Row}(A)\},
$$
where $\text{Col}(M)$ and $\text{Row}(M)$ are column and row spaces of matrix $M$ respectively.
\end{lemma}

\begin{proof}
    See Appendix \ref{app:proof_lemma_subspace}.
\end{proof}
\end{tcolorbox}
We identify three key limitations, which arise due to this and otherwise:

\noindent
1) \textbf{Inadequate Gradient Approximation:}  \label{subsec:Motiv_Pro} \lora optimization is mathematically equivalent to full FT using a constrained low-rank gradient. The gradient of \lora does not optimally approximate the full gradient, and needs to be tuned at each step. LoRA-Pro \cite{LoRA-Pro} finds that this results in suboptimal performances, and provides a closed form solution to optimize the gradients. In LoRA-XS, the gradient updates are restricted to an even more constrained low-rank space since $A$ and $B$ are fixed. We posit that the limitation becomes particularly severe when the ideal updates lie outside the space spanned by fixed $A$ and $B$, and consequently has a larger impact on performance.

\noindent
2) \textbf{Suboptimal Initialization:} \label{subsec:Motiv_Init} While initialization impacts all low-rank methods, it becomes critical in \xs where $A$ and $B$ are frozen. Unlike LoRA where poor initialization can be compensated through training, LoRA-XS relies entirely on its initial subspace defined by $A$ and $B$. Consider the zero initialization of the $B$ matrix, for example. While \lora may experience some performance degradation in this case \cite{lora_ga, pissa}, the ideal low-rank update $\Delta W$ can still be reached through gradient descent. In fact, zero initialization for the $B$ matrix is commonly used, including in the original \lora paper \cite{lora}. However, in LoRA-XS, this results in no learning, as the product $BRA$ remains zero. \xs uses the most significant subspaces spanned by the columns of pre-trained weights for initialization, inspired by PiSSA \cite{pissa}. This initialization is not aligned well with FT because it fails to capture the specific subspaces relevant to the FT task. 
% This issue, known to affect \lora’s performance \cite{lora_ga}, could have an even greater impact on \xs for the reasons discussed above.

\noindent
3) \textbf{Scaling Factor Sensitivity:} \label{subsec:Motiv_hyper} The scaling factor $s$, present in almost every \lora based method, requires tuning to maintain stability during training. This factor acts as a bridge between the low-rank and full-rank spaces, compensating for the dimensional mismatch in gradients. Poor tuning of $s$ can lead to unstable training or slow convergence (rsLoRA~\citep{rslora}), adding complexity and potentially limiting practical deployment.

% \textbf{1) The low-rank gradient approximations inadequately capture the full FT updates.} \lora optimization is mathematically equivalent to full FT using a constrained low-rank gradient. Building on this observation, LoRA-Pro \cite{LoRA-Pro} shows that LoRA's performance gap stems from its implicit low-rank gradient failing to properly approximate the full gradient during optimization, and provides a closed-form solution to optimize this gradient. We posit that \xs exhibits similar limitations, as its gradient updates are also constrained to a low-rank space. However, our case presents unique challenges since only the middle matrix $R$ is trainable, while both $A$ and $B$ are fixed. We propose a method to optimally approximate the full FT gradient within these stricter constraints.\\

% \textbf{2) The initialization using the pre-trained weight matrix $W_0$ and fixed matrices $A$ and $B$ proves suboptimal for target domain adaptation.}
% We examine each limitation in detail and propose independent solutions, culminating in a unified approach \ours

% \textbf{3) Need to tune hyperparameter $\alpha$}. 

\subsection{Approximation of the full FT gradient}\label{subsec:Pro}

% \lora optimization is mathematically equivalent to full FT using a constrained low-rank gradient. Building on this observation, LoRA-Pro \cite{LoRA-Pro} shows that LoRA's performance gap stems from its implicit low-rank gradient failing to properly approximate the full gradient during optimization, and provides a closed-form solution to optimize this gradient. We posit that \xs exhibits similar limitations, as its gradient updates are also constrained to a low-rank space. However, our case presents unique challenges since only the middle matrix $R$ is trainable, while both $A$ and $B$ are fixed. We propose a method to optimally approximate the full FT gradient within these stricter constraints.

% Consider a weight matrix $W \in \mathbb{R}^{m \times n}$ in the neural network. In \xs, we reparameterize the weight update as:
% $W = W_0 + BRA$
% where $W_0 \in \mathbb{R}^{m \times n}$ is the pre-trained weight matrix, $B \in \mathbb{R}^{m \times r}$ and $A \in \mathbb{R}^{r \times n}$ are fixed matrices, and $R \in \mathbb{R}^{r \times r}$ is the only trainable matrix, with $r \ll \min(m,n)$.
% In full FT, the weight updates are driven by the gradient $g = \frac{\partial \mathcal{L}}{\partial W}$. However, in \xs, we only train $R$, resulting in a constrained low-rank update. The \xs gradient of R is given by  $g^R_{LoRA-XS} = \frac{\partial \mathcal{L}}{\partial R}$

% Building on insights from LoRA-Pro \cite{LoRA-Pro}, we observe that \xs optimization is mathematically equivalent to using a low-rank gradient for updates.
As mentioned, \lora optimization is equivalent to full FT using a constrained low-rank gradient. However, the update generated using the gradients of \lora does not result in the same update which the low-rank gradient would have generated. The following holds true for \xs as well.
To understand this, let us look at the change in weight $W$ and its relationship with changing of low-rank matrix $R$, which can be simply given by  
$ \dd W = -sB (\dd R) A. $ This implies that updating $R$ with gradient $g^R$ is equivalent to updating $W$ with low rank equivalent gradient $\tilde{g}$ in full FT (Definition \ref{def:equi_grad}).
% \begin{gather}
% \tilde{g} = \frac{\partial W}{\partial R} g^R = sBg^R A
% \end{gather}
\begin{tcolorbox}[
    colback=cyan!10,
    colframe=black,
    % inner padding
    boxsep=1pt,      % space around the content
    left=6pt,        % additional left padding
    right=6pt,       % additional right padding
    top=6pt,         % additional top padding
    bottom=6pt,      % additional bottom padding
    % vertical space before/after the box
    before skip=6pt, % space above the box
    after skip=6pt   % space below the box
]
\begin{definition*}
\label{def:equi_grad}
We define the equivalent gradient in LoRA-XS as:
$
\tilde{g} = sBg^R A
$,
where $g^R$ is the gradient of $L$ with respect to $R$.
\end{definition*}
\end{tcolorbox}

The equivalent gradient describes the virtual low-rank gradient of matrix $W$ in \xs optimization process, despite $W$ not being directly trainable. This gradient determines how updates to $R$ affect $W$.
To bridge the performance gap between \xs and full FT, we aim to minimize the discrepancy between the equivalent gradient $\tilde{g}$ and the full gradient $g$. First, we establish the relationship between gradients in \xs optimization in Lemma \ref{lemma:gR}.
\begin{tcolorbox}[
    colback=cyan!10,
    colframe=black,
    % inner padding
    boxsep=1pt,      % space around the content
    left=4pt,        % additional left padding
    right=4pt,       % additional right padding
    top=6pt,         % additional top padding
    bottom=6pt,      % additional bottom padding
    % vertical space before/after the box
    before skip=6pt, % space above the box
    after skip=6pt   % space below the box
]
\begin{lemma}
\label{lemma:gR}
The gradient of the loss with respect to matrix $R$ can be expressed in terms of the gradient with respect to the weight matrix $W$ as:
$
g^R_{LoRA-XS} = s B^\top  g A^\top 
$.
\end{lemma}
\begin{proof}
    See Appendix \ref{app:proof_lemma_gr}.
\end{proof}
\end{tcolorbox}
We now formulate our objective to minimize the distance between the equivalent gradient and the full gradient. We do not have access to the full FT gradient $g$ during \xs based FT. Thus we need to find the ideal gradient with respect to $R$, given by $g^R$, and subsequently the optimal approximation $\Tilde{g}$, in terms of the gradient which is available to us during training: $g^R_{LoRA-XS}$. Fortunately, this optimization problem admits a closed-form solution independent of $g$ as described in Theorem \ref{theorem:pro}.
% \begin{tcolorbox}[colback=cyan!10,colframe=black]
% \begin{theorem}\label{theorem:pro}
% Assume matrices $B \in \mathbb{R}^{m \times r}$ and $A \in \mathbb{R}^{r \times n}$ are both full rank. For the objective  $\text{min}_{g^R} ||\tilde{g} - g||^2_F$, such that $\Tilde{g} = sB g^R A$,  the optimal solution is given by:
% \begin{equation}
% g^R = \dfrac{1}{s^2} (B^\top  B)^{-1} g^R_{LoRA-XS} (A A^\top )^{-1}
% \end{equation}
% \end{theorem}

\begin{tcolorbox}[
    colback=cyan!10,
    colframe=black,
    % inner padding
    boxsep=1pt,      % space around the content
    left=6pt,        % additional left padding
    right=6pt,       % additional right padding
    top=6pt,         % additional top padding
    bottom=6pt,      % additional bottom padding
    % vertical space before/after the box
    before skip=6pt, % space above the box
    after skip=6pt   % space below the box
]
\begin{theorem}\label{theorem:pro}
For full-rank $A$ and $B$ matrices, the optimal solution for the objective \\$\text{min}_{g^R} ||\tilde{g} - g||^2_F$, such that $\Tilde{g} = sB g^R A$, is:
$
g^R = \dfrac{1}{s^2} (B^\top  B)^{-1} g^R_{LoRA-XS} (A A^\top )^{-1}
$.
\end{theorem}
\begin{proof}
   See Appendix \ref{app:proof_pro}.
\end{proof}
\end{tcolorbox}
The closed-form solution in Theorem \ref{theorem:pro} solves the optimization problem $\text{min}_{g^R} ||\tilde{g} - g||^2_F$, but by itself doesn't ensure the loss will decrease when updating $R$. Through Theorem \ref{theorem:loss-neg}, we prove that the change in loss is non-positive ($\Delta L \le 0$). 
This property is fundamental to optimization as it guarantees consistent loss minimization throughout training.

% \begin{tcolorbox}[colback=cyan!10,colframe=black]
% \begin{theorem}\label{theorem:loss-neg}
% Consider the update for matrix $R$ using the solution derived in Theorem \ref{theorem:pro}:
% \begin{gather*}
% R\leftarrow R - \eta g^R
% \end{gather*}
% where $\eta \ge 0$ is the learning rate.
% This update guarantees a reduction in the loss function, similar to traditional gradient descent methods. Specifically, the change in loss $\dd L$ is given by:
% \begin{gather*}
% \dd L = -\eta \langle g^R_{LoRA-XS}, g^R \rangle
% \le 0
% \end{gather*}
% \begin{proof}
% Appendix \ref{app:proof_loss-neg}.
% \end{proof}
% \end{theorem}
% \end{tcolorbox}

\begin{tcolorbox}[
    colback=cyan!10,
    colframe=black,
    % inner padding
    boxsep=1pt,      % space around the content
    left=6pt,        % additional left padding
    right=6pt,       % additional right padding
    top=6pt,         % additional top padding
    bottom=6pt,      % additional bottom padding
    % vertical space before/after the box
    before skip=6pt, % space above the box
    after skip=6pt   % space below the box
]
\begin{theorem}\label{theorem:loss-neg}
Consider the update for matrix $R$ using the solution derived in Theorem \ref{theorem:pro}:\\
$
R\leftarrow R - \eta g^R
$,
where $\eta > 0$ is the (sufficiently small) learning rate.
This update guarantees a reduction in the loss $\Delta L$, given by:
$
\Delta L 
% = \langle \frac{\partial L}{\partial R}, \dd R \rangle_F 
= -\eta \langle g^R_{LoRA-XS}, g^R \rangle_F + o(\eta)
\le 0.
$
\begin{proof}
See Appendix \ref{app:proof_loss-neg}.
\end{proof}
\end{theorem}
\end{tcolorbox}

\subsection{Initialization using update approximation} \label{subsec:Init}

 % Initialization of weight matrices and adapters plays a vital role in the performance of models. For most networks however, a non optimum initialization can be theoretically overcome since the weight matrices till have access to the entire subspaces. As an example, consider the case of \lora where while a bad initialization definitely impacts results, it is obvious that an ideal $B_{ideal}$ and $A_{ideal}$ lies in the same solution space as a good and bad initalization. This however is not true in case of \xs, where the matrices $B$ and $A$ are frozen. So an ideal update $\Delta W_{ideal}$ even if it is low-rank, may never be attained for for any value of $R$. This makes the problem of initialization critical in \xs, as compared to \lora for example where any low-rank ideal update $\Delta W_{ideal}$ can be learnt by an appropriate combination of $B$ and $A$ adapters. The question now is what low rank subspace $\mathbb{R}^{r \times r} $ of full finetuning space $\mathbb{R}^{m \times m} $ does the ideal update lie in, if at all in any. We postulate that the best initialization that $B$ ,$R$ and $A$ is such that it simulates the full FT, high rank update for the first step. Now suppose the full FT update is given by $\Delta W_{first-step}$, then our ideal initialization should follow :

 In FT, the primary goal is to update weights to better suit the target task. The initial gradient steps are particularly informative, as they indicate the direction of desired adaptation. We leverage this insight by using the first update step from full FT for initialization.

This approach offers two key advantages. First, it ensures the low-rank space captures the most relevant subspace for the target task rather than relying on pre-trained properties. Second, since $A$ and $B$ are fixed, initializing them to span the subspace of early adaptation increases the likelihood of capturing useful updates throughout training. This also ensures that the final update is learnt in the correct subspace, of which we have no apriori information besides the first full FT step.
Our method is summarized as: set such initialization that best approximates the first step of full FT. Given a full FT update $\Delta W_{first-step}$, our initialization satisfies:
\begin{gather}
    s B_{init}R_{init}A_{init} \approx \Delta W_{first-step}
\end{gather}
 The first step of full FT, for Adam-based optimizers such as AdamW, for sample $x_i$ is: 
\begin{gather}
     \Delta W_{first-step} =  -\eta \times \textbf{sign}(\nabla_W \mathcal{L}(W_0,x_i))
\end{gather}
% In general, a better update without taking multiple full FT steps would be to take the gradient averaged across different samples. We estimate the gradient direction by averaging the gradients of losses of $n$ samples from the dataset $\mathbb{X}$, and then taking the first step
However, the usage of a single sample may lead to noisy estimates. Instead, we compute a more stable initialization by averaging gradients over a subset of the training data:
\begin{gather}
\Delta W_{avg} = -\eta \textbf{sign}(  \sum_{i=0}^{n \leq |\mathbb{X}| } \nabla_W \mathcal{L}(W_0,x_i)),\quad x_i \in \mathbb{X}
\end{gather}
% Now we can write best approximate the estimated gradient $\Delta W_{avg}$ by using Truncated SVD. The optimality of the approximation is proven using Eckart-Young theorem. 
Since AdamW is used as the optimizer for both full FT and LoRA-SB training, we approximate its first update step using the sign of the summed gradients rather than their raw values (see Appendix \ref{app:adam_sign} for details).
This better captures the direction of adaptation required for the target task while being less sensitive to individual sample variations. 
We then use truncated SVD to obtain a low-rank approximation of $\Delta W_{\text{avg}}$, and express it as $sBRA$. 
There exist infinite combinations of $B$ and $A$ which can obey this relationship. 
For instance, we can initialize $B$ and $A$ as $US$ and $V^\top $ and keep $R$ as $I/s$. 
This is equivalent to the $B$ and $A$ initialization in \xs but by approximating the update rather than the pre-trained matrix.
% \begin{gather}
%     U,S,V^\top  \leftarrow \textbf{SVD}(\Delta W_{avg}) \label{W_svd}\\
%     B_{init} \leftarrow  U[1:r] \label{B_init}\\
%     A_{init} \leftarrow  V[1:r] \label{A_init}\\
%     R_{init} \leftarrow  S[1:r,1:r]/s \label{S_init}   
% \end{gather}
 % By the Eckart-Young theorem, this gives the optimal rank-$r$ approximation of the initial full FT update. 
 The above process can be computed for any optimizer, by approximating the corresponding first step.
 We compute this specifically for AdamW since we use it.

% Another heuristic which might lead to a good initialization is setting the weights $B$ and $A$, such that they match the first update also approximately matches the direction of  $\Delta W$.
% \begin{gather}
% \Delta(s B_{init} R_{init}A_{init}) \approx \gamma \Delta W 
% \end{gather}
% Thankfully, we don't have to choose between the two.   For SGD, we prove that setting $B_{init}$ and $A_{init}$ using Truncated SVD of the first update, results in the first update of \xs to best approximate the direction of the update of full FT. Formally

% \begin{tcolorbox}[colback=cyan!10,colframe=black]
% \begin{theorem}
% If $A_{init}$ and $B_{init}$ are initialized using the first    SGD step of full FT, then
% \begin{gather}
% \Delta(s B_{init} R_{init}A_{init}) \approx \gamma \Delta W \dots \text{for } \gamma =  s^2
% \end{gather}
% \end{theorem}
% \end{tcolorbox}

\subsection{Scaling Factor independence} \label{subsec:Hyper}

The hyperparameter $\alpha$ is used in LoRA and other decomposition-based methods to tackle instability caused to improper scaling of the updates. The gradient scaling is accounted for, by adding a hyperparameter to normalize the updates. The importance of scaling is shown in methods like rank stabilization \cite{rslora}. However, the full FT gradient $g$ needs no such tuning. We claim that approximating the full FT gradient removes the need for introducing a scaling factor, as shown in Theorem \ref{theorem:hyper}. 

\begin{tcolorbox}[
    colback=cyan!10,
    colframe=black,
    % inner padding
    boxsep=1pt,      % space around the content
    left=6pt,        % additional left padding
    right=6pt,       % additional right padding
    top=6pt,         % additional top padding
    bottom=6pt,      % additional bottom padding
    % vertical space before/after the box
    before skip=6pt, % space above the box
    after skip=6pt   % space below the box
]
\begin{theorem} \label{theorem:hyper}
The equivalent gradient $\Tilde{g}$ is hyperparameter $s$ independent for $\Tilde{g}=sBg^RA$, but not for
$  \Tilde{g}=sBg^R_{LoRA-XS}A$.
\end{theorem}
\begin{proof}
   See Appendix \ref{app:proof_hyper}.
\end{proof}
\end{tcolorbox}

The scaling factor independence of the equivalent gradient eliminates the need for manual gradient scaling. Updates to $W$ depend solely on this gradient (modulo learning rate), making any additional scaling redundant. 
This can be understood by examining the relationship with the full FT gradient $g$. Since $g$ is naturally scaled for optimal weight updates, and our method approximates $g$ in a constrained subspace, the equivalent gradient inherits appropriate scaling automatically. This property is unique to our gradient approximation approach and does not hold for standard LoRA-XS.

\subsection{LoRA-SB: Update approximation initialization is a \textit{silver bullet} } \label{subsec:LoRA_SB}

The solutions discussed independently address the gradient approximation and initialization problems, while also providing scaling factor independence.
LoRA-SB, elegantly combines these solutions through a simple initialization strategy, derived from approximating the first full FT step:
\begin{gather}
    U,S,V^\top  \leftarrow \textbf{SVD}(\Delta W_{avg}) \label{W_svd}\\
    B_{init} \leftarrow  U[1:r],
    A_{init} \leftarrow  V[1:r],
    R_{init} \leftarrow  \dfrac{1}{s}S[1:r,1:r] \label{S_init}   
\end{gather}
By the Eckart-Young theorem \citep{eckart1936approximation, mirsky1960symmetric}, this gives the optimal rank-$r$ approximation of the full FT update. 
where $U$, $S$, $V$ are obtained from truncated SVD of the averaged first update $\Delta W_{\text{avg}}$. 
%This initialization leads to several key advantages that address the problems identified earlier.
This initialization leads to several key advantages.

\textbf{Simplified Gradient Optimization.} 
Our initialization ensures $B_{\text{init}}$ and $A_{\text{init}}$ form orthonormal bases in $\mathbb{R}^m$ and $\mathbb{R}^n$ respectively, leading to $B^\top B = AA^\top  = I$. With fixed $B$ and $A$ matrices being orthonormal, the need for complex matrix inversions during training is eliminated, , as the optimal update step, derived in Equation \ref{theorem:pro}, simplifies to:
\begin{equation*}
g^R =  \dfrac{1}{s^2} (B^\top  B)^{-1} g^R_{LoRA-XS} (A A^\top )^{-1}=\frac{1}{s^2}g^R_{LoRA-XS} \label{eq:show}
\end{equation*}

\textbf{Optimal Update Approximation.}
Our initialization guarantees that the first update optimally approximates the full FT weight updates:
$
sB_{\text{init}}R_{\text{init}}A_{\text{init}}  \approx \Delta W_{avg}.
$
By the Eckart-Young theorem, this is the optimal rank-$r$ approximation of the initial full FT update. 

\textbf{Scaling Factor Independence.} 
As shown in Theorem \ref{theorem:hyper}, when gradient approximation is applied with orthonormal $B$ and $A$, the hyperparameter $s$ can be set to 1, resulting in guaranteed optimal gradient approximation at every step, without requiring any scaling factor:
\begin{equation}
\boxed{
g^R = g^R_{\text{LoRA-XS}}
}
\end{equation}
\textbf{Guaranteed Loss Reduction.}
Since $B$ is a tall orthonormal and $A$ a wide orthonormal matrix, they remain full rank throughout training. This ensures that $dL$ remains negative (Theorem \ref{theorem:loss-neg}), guaranteeing stable optimization and convergence.
\begin{gather}
\Delta(s B_{init} R_{init}A_{init}) \approx \gamma \Delta W \label{eq:heuristic}
\end{gather}
Another heuristic which might lead to a good initialization is setting $B$ and $A$, such that the first update also approximately matches the $\Delta W$ direction (Equation \ref{eq:heuristic}).
Thankfully, we don't have to choose between the two. 
For SGD, we prove that setting $B_{init}$ and $A_{init}$ using Equations \ref{W_svd}-\ref{S_init}, results in the first update of \xs to best approximate the direction of the full FT update (Theorem \ref{theorem:grad_init}). 

\begin{tcolorbox}[
    colback=cyan!10,
    colframe=black,
    % inner padding
    boxsep=1pt,      % space around the content
    left=6pt,        % additional left padding
    right=6pt,       % additional right padding
    top=6pt,         % additional top padding
    bottom=6pt,      % additional bottom padding
    % vertical space before/after the box
    before skip=6pt, % space above the box
    after skip=6pt   % space below the box
]
\begin{theorem} \label{theorem:grad_init}
If $A_{init}$ and $B_{init}$ are initialized using \ours for the first step of SGD optimizer, then the update given by LoRA-SB, $\Delta(B_{init} R_{init}A_{init})$ , is the best low-rank approximation of full fine-tuning update,
$
 \Delta W 
$.
\end{theorem}
\begin{proof}
    See Appendix \ref{app:proof_grad_init}.
\end{proof}
\end{tcolorbox}

While Theorem \ref{theorem:grad_init} is stated for SGD, the result extends to other SGD-based optimizers such as AdamW.
In practice, we use AdamW and approximate the first update by taking the sign of the averaged gradients, consistent with AdamW’s first-step behavior. This produces an initialization whose SVD still yields the optimal rank-$r$ approximation of the simulated full FT update.

\textbf{Initialization Memory.}
%We present a PyTorch-like implementation of our method in Algorithm \ref{alg:lora-sb} (Appendix \ref{app:algo}). 
To optimize GPU memory during \textbf{initialization}, we hook into the backward pass and compute the gradients layerwise, immediately discarding the computed gradients \citep{lv2024parameterfinetuninglargelanguage, lora_ga}. 
This ensures \( O(1) \) memory usage, independent of the number of layers, keeping GPU memory well within limits.
This guarantees that the memory required for LoRA-SB initialization never exceeds the memory needed for subsequent LoRA-SB fine-tuning, and that \textbf{the peak memory usage of the entire LoRA-SB algorithm never exceeds that of standard LoRA} and other baselines.

\textbf{LoRA-SB Advantages over LoRA.}
Many properties described above are not achievable with standard LoRA methods. Even if $B$ and $A$ are initialized as orthonormal in LoRA, subsequent updates do not preserve this property because $B$ and $A$ are trainable. This results in several challenges in using LoRA (even with optimal gradient approximation) compared to LoRA-SB:
\vspace{-0.5em}
\begin{itemize}[leftmargin=*,itemsep=0pt]
    \item Potential instability of $(B^\top  B)^{-1}$ and $(A A^\top )^{-1}$, not guaranteed to remain non-singular throughout. 
    \item Inability to ensure consistent loss reduction due to potential rank deficiency, $B$ and $A$ may not remain full-rank throughout training. 
    \item Necessity to fine-tune the scaling factor hyperparameter $\alpha$. 
    \item Repeated re-computation of $B^\top  B$ and $A A^\top $ is required at each optimizer step for accurate gradient approximation. 
\end{itemize}
\vspace{-0.5em}
% LoRA-SB's simple initialization strategy solves multiple problems simultaneously, offering a robust and efficient approach to low-rank adaptation.

%Memory usage can be further optimized through gradient accumulation and quantization.

%% file: experiments.tex
\section{Experiments} \label{sec:exps}

We evaluate over $16$ different datasets on $3$ widely-used benchmarks, using models ranging from the 355 M RoBERTa-large model to the 9 B Gemma-2 model. Our setup spans both masked and autoregressive architectures, allowing us to comprehensively assess the effectiveness of LoRA-SB. 
Specifically, we fine-tune RoBERTa-large \citep{liu2019robertarobustlyoptimizedbert}, Llama-3.2 3B \citep{llama3}, Mistral-7B \citep{mistral7b}, and Gemma-2 9B \citep{gemma2}.
\textbf{We compute the update approximation using only $\mathbf{1/1000}$ ($0.1\%$) of each dataset's total size}. This ensures that the training time overhead is minimal and has a negligible effect on efficiency. 
Detailed hyperparameter and dataset details are given in Appendix \ref{app:exps} and \ref{app:datasets}, respectively.

\textbf{Baselines.}
We compare LoRA-SB against full FT, LoRA \citep{lora}, LoRA-XS \citep{LoRA-XS}, and several popular variants of LoRA - rsLoRA \citep{rslora}, PiSSA \citep{pissa}, DoRA \citep{Liu_Wang_Yin_Molchanov_Wang_Cheng_Chen_2024}, and LoRA-Pro \citep{LoRA-Pro}. 
% rsLoRA introduces a rank-scaled stabilization factor ($\alpha/\sqrt{r}$) to enhance stability, while PiSSA updates only the principal components of the pre-trained weight $W$ and freezes the residuals.

\begin{table}[!h]
\centering
\caption{
    Comparison of FT methods on Mistral-7B and Gemma-2 9B across arithmetic benchmarks. \# Params denotes the number of trainable parameters. Best results among PEFT methods are in \textbf{bold}.
}
\setlength{\tabcolsep}{5.8pt}
\small
\begin{tabular}{l c|ccc|ccc}
  \toprule
  \multirow{2}{*}{\textbf{Method}} & \multirow{2}{*}{\textbf{Rank}} & \multicolumn{3}{c|}{\textbf{Mistral-7B}} & \multicolumn{3}{c}{\textbf{Gemma-2 9B}} \\
  \cmidrule{3-8}
  & & \textbf{\# Params} & \textbf{GSM8K} ($\uparrow$) & \textbf{MATH} ($\uparrow$) & \textbf{\# Params} & \textbf{GSM8K} ($\uparrow$) &\textbf{ MATH} ($\uparrow$) \\
  \midrule
  Full FT   & -  & $7.24$ B   & $63.87$ & $17.65$ & $9.24$ B   & $79.23$ & $38.02$ \\
  LoRA      & 32 & $83.88$ M & $61.94$ & $15.98$ & $108.04$ M & $76.19$ & $36.56$ \\
  rsLoRA    & 32 & $83.88$ M & $62.15$ & $16.24$ & $108.04$ M & $76.84$ & $36.88$ \\
  PiSSA     & 32 & $83.88$ M & $62.43$ & $16.52$ & $108.04$ M & $77.12$ & $37.04$ \\
  DoRA    & 32 & $85.26$ M & $62.65$ & $16.64$ & $109.88$ M & $77.58$ & $37.04$ \\
  LoRA-Pro     & 32 & $83.88$ M & $63.07$ & $17.32$ & $108.04$ M & $78.26$ & $37.53$ \\
  \cmidrule{1-8}
  LoRA-XS   & 32 & $0.23$ M  & $54.28$ & $13.36$ & $0.30$ M   & $74.07$ & $34.62$ \\
  LoRA-XS   & 64 & $0.92$ M  & $57.08$ & $15.62$ & $1.20$ M   & $75.02$ & $36.46$ \\
  LoRA-XS   & 96 & $2.06$ M  & $58.53$ & $16.42$ & $2.71$ M   & $75.21$ & $36.98$ \\
  \cmidrule{1-8}
  \cellcolor{cyan!10}LoRA-SB   & \cellcolor{cyan!10}32 & \cellcolor{cyan!10}$0.23$ M & \cellcolor{cyan!10}$58.91$ & \cellcolor{cyan!10}$15.28$ & \cellcolor{cyan!10}$0.30$ M & \cellcolor{cyan!10}$75.44$ & \cellcolor{cyan!10}$36.66$ \\
  \cellcolor{cyan!10}LoRA-SB   & \cellcolor{cyan!10}64 & \cellcolor{cyan!10}$0.92$ M & \cellcolor{cyan!10}$60.73$ & \cellcolor{cyan!10}$16.28$ & \cellcolor{cyan!10}$1.20$ M & \cellcolor{cyan!10}$76.65$ & \cellcolor{cyan!10}$37.14$ \\
  \cellcolor{cyan!10}LoRA-SB   & \cellcolor{cyan!10}96 & \cellcolor{cyan!10}$2.06$ M & \cellcolor{cyan!10}$\mathbf{63.38}$ & \cellcolor{cyan!10}$\mathbf{17.44}$ & \cellcolor{cyan!10}$2.71$ M & \cellcolor{cyan!10}$\mathbf{78.40}$ & \cellcolor{cyan!10}$\mathbf{37.70}$ \\
  \bottomrule
\end{tabular}
\label{tab:arithmetic}
\end{table}

\begin{table}[!h]
\centering
\caption{
    Comparison of FT methods on Llama-3.2 3B across eight commonsense reasoning datasets. \# Params denotes the number of trainable parameters. Best results among PEFT methods are in \textbf{bold}.}
\setlength{\tabcolsep}{2.5pt}
\small
\begin{tabular}{lcc|ccccccccc}
  \toprule
  \multirow{2}{*}{\textbf{Method}} & \multirow{2}{*}{\textbf{Rank}} & \multirow{2}{*}{\textbf{\# Params}} & \multicolumn{9}{c}{\textbf{Accuracy ($\uparrow$)}} \\
  \cmidrule{4-12}
  & & & \textbf{BoolQ} & \textbf{PIQA} & \textbf{SIQA} & \textbf{HellaS.} & \textbf{WinoG.} & \textbf{ARC-e} & \textbf{ARC-c} & \textbf{OBQA} & \textbf{Avg.} \\
  \midrule
  Full FT & - & $3.21$ B & $70.43$ & $85.64$ & $80.45$ & $91.92$ & $85.02$ & $88.52$ & $75.29$ & $81.88$ & $82.39$ \\
  LoRA & $32$ & $48.63$ M & $70.03$ & $85.20$ & $79.12$ & $90.71$ & $82.24$ & $86.91$ & $74.32$ & $\mathbf{81.87}$ & $81.30$ \\
  rsLoRA & $32$ & $48.63$ M & $69.81$ & $85.63$ & $78.92$ & $90.45$ & $82.02$ & $86.71$ & $74.18$ & $81.72$ & $81.11$ \\
  PiSSA & $32$ & $48.63$ M & $70.12$ & $85.42$ & $79.44$ & $90.88$ & $82.68$ & $87.23$ & $74.61$ & $81.79$ & $81.52$ \\
  
  DoRA  & $32$  & $49.40$ M  & $70.43$ & $85.63$& $79.68$ & $90.76$ & $82.90$ & $87.61$ & $74.87$ & $82.04$ & $81.74$ \\
  LoRA-Pro  & $32$ & $48.63$ M  & $\mathbf{71.28}$ & $\mathbf{85.81}$ & $79.35$ & $90.90$ & $83.42$ & $87.24$ & $\mathbf{75.32}$ & $81.74$ & $81.88$ \\
  \cmidrule{1-12}
  LoRA-XS & $32$ & $0.20$ M & $65.01$ & $82.87$ & $76.17$ & $87.32$ & $80.12$ & $84.78$ & $70.31$ & $75.71$ & $77.79$ \\
  LoRA-XS & $64$ & $0.80$ M & $66.53$ & $83.12$ & $77.98$ & $88.53$ & $81.76$ & $85.15$ & $72.04$ & $77.14$ & $79.03$ \\
  LoRA-XS & $96$ & $1.81$ M & $67.28$ & $83.35$ & $78.66$ & $88.99$ & $82.08$ & $85.18$ & $72.61$ & $78.88$ & $79.63$ \\
  \cmidrule{1-12}
  \cellcolor{cyan!10}LoRA-SB & \cellcolor{cyan!10}$32$& \cellcolor{cyan!10}$0.20$ M & \cellcolor{cyan!10}$66.33$ & \cellcolor{cyan!10}$84.06$ & \cellcolor{cyan!10}$78.91$ & \cellcolor{cyan!10}$89.04$ & \cellcolor{cyan!10}$81.37$ & \cellcolor{cyan!10}$86.62$ & \cellcolor{cyan!10}$72.44$ & \cellcolor{cyan!10}$76.97$ & \cellcolor{cyan!10}$79.47$ \\
  \cellcolor{cyan!10}LoRA-SB & \cellcolor{cyan!10}$64$ & \cellcolor{cyan!10}$0.80$ M & \cellcolor{cyan!10}$68.35$ & \cellcolor{cyan!10}$84.55$ & \cellcolor{cyan!10}$79.94$ & \cellcolor{cyan!10}$\mathbf{91.68}$ & \cellcolor{cyan!10}$83.03$ & \cellcolor{cyan!10}$87.84$ & \cellcolor{cyan!10}$74.83$ & \cellcolor{cyan!10}$80.12$ & \cellcolor{cyan!10}$81.29$ \\
  \cellcolor{cyan!10}LoRA-SB & \cellcolor{cyan!10}$96$ & \cellcolor{cyan!10}$1.81$ M & \cellcolor{cyan!10}$70.34$ & \cellcolor{cyan!10}$84.76$ & \cellcolor{cyan!10}$\mathbf{80.19}$ & \cellcolor{cyan!10}$91.62$ & \cellcolor{cyan!10}$\mathbf{84.61}$ & \cellcolor{cyan!10}$\mathbf{87.92}$ & \cellcolor{cyan!10}$74.74$ & \cellcolor{cyan!10}$81.20$ & \cellcolor{cyan!10}$\mathbf{81.92}$ \\
  \bottomrule
\end{tabular}
\label{tab:cr}
\end{table}

\begin{table}[!h]
\centering
\caption{
    Comparison of FT methods on RoBERTa-large across GLUE datasets. \# Params denotes the number of trainable parameters. Best results among PEFT methods are in \textbf{bold}. We use Pearson correlation for STS-B, Matthew's correlation for CoLA, and accuracy for others.
}
\setlength{\tabcolsep}{4pt}
\small
\begin{tabular}{lcc|ccccccc}
  \toprule
  \multirow{2}{*}{\textbf{Method}} & \multirow{2}{*}{\textbf{Rank}} & \multirow{2}{*}{\textbf{\# Params}} & \textbf{CoLA} & \textbf{RTE} & \textbf{MRPC} & \textbf{STS-B} & \textbf{QNLI} & \textbf{SST-2} & \textbf{All} \\
  & & & \textbf{Mcc} $\uparrow$ & \textbf{Acc} $\uparrow$ & \textbf{Acc} $\uparrow$ & \textbf{Corr} $\uparrow$ & \textbf{Acc} $\uparrow$ & \textbf{Acc} $\uparrow$ & \textbf{Avg.} $\uparrow$ \\
  \midrule
  Full FT & - & $355.36$ M & $68.44$ & $83.42$ & $90.21$ & $91.76$ & $93.92$ & $96.21$ & $87.33$ \\
  LoRA & $8$ & $2162.69$ K & $68.02$ & $82.98$ & $90.05$ & $91.43$ & $93.42$ & $95.98$ & $86.98$ \\
  rsLoRA & $8$ & $2162.69$ K & $67.87$ & $82.84$ & $89.97$ & $91.30$ & $93.29$ & $95.87$ & $86.85$ \\
  PiSSA & $8$ & $2162.69$ K & $68.22$ & $83.14$ & $90.10$ & $91.59$ & $93.55$ & $96.03$ & $87.10$ \\
  DoRA & $8$ & $2260.99$ K & $68.05$ & $83.04$ & $89.93$ & $91.34$ & $93.11$ & $95.82$ & $86.88$ \\
  LoRA-Pro & $8$ & $2162.69$ K & $67.98$ & $\mathbf{83.40}$ & $\mathbf{90.49}$ & $91.38$ & $93.37$ & $95.98$ & $87.10$ \\
  \cmidrule{1-10}
  LoRA-XS & $8$ & $6.14$ K & $61.07$ & $75.23$ & $86.21$ & $89.29$ & $92.44$ & $94.72$ & $83.16$ \\
  LoRA-XS & $16$ & $24.57$ K & $63.32$ & $79.06$ & $86.28$ & $90.36$ & $93.69$ & $95.76$ & $84.70$ \\
  LoRA-XS & $24$ & $55.20$ K & $66.27$ & $80.14$ & $88.48$ & $90.77$ & $93.21$ & $95.89$ & $85.79$ \\
  \cmidrule{1-10}
  \cellcolor{cyan!10}LoRA-SB & \cellcolor{cyan!10}$8$ & \cellcolor{cyan!10}$6.14$ K & \cellcolor{cyan!10}$63.57$ & \cellcolor{cyan!10}$78.43$ & \cellcolor{cyan!10}$88.72$ & \cellcolor{cyan!10}$90.59$ & \cellcolor{cyan!10}$92.95$ & \cellcolor{cyan!10}$95.07$ & \cellcolor{cyan!10}$84.88$ \\
  \cellcolor{cyan!10}LoRA-SB & \cellcolor{cyan!10}$16$ & \cellcolor{cyan!10}$24.57$ K & \cellcolor{cyan!10}$64.36$ & \cellcolor{cyan!10}$82.31$ & \cellcolor{cyan!10}$89.71$ & \cellcolor{cyan!10}$91.24$ & \cellcolor{cyan!10}$\mathbf{93.89}$ & \cellcolor{cyan!10}$95.87$ & \cellcolor{cyan!10}$86.23$ \\
  \cellcolor{cyan!10}LoRA-SB & \cellcolor{cyan!10}$24$ & \cellcolor{cyan!10}$55.20$ K & \cellcolor{cyan!10}$\mathbf{68.28}$ & \cellcolor{cyan!10}$83.03$ & \cellcolor{cyan!10}$90.12$ & \cellcolor{cyan!10}$\mathbf{91.65}$ & \cellcolor{cyan!10}$93.75$ & \cellcolor{cyan!10}$\mathbf{96.11}$ & \cellcolor{cyan!10}$\mathbf{87.16}$ \\
  \bottomrule
\end{tabular}
\label{tab:glue}
\end{table}

\subsection{Arithmetic Reasoning}

We fine-tune Mistral-7B \citep{mistral7b} and Gemma-2 9B \citep{gemma2} on 50K samples from MetaMathQA \citep{metamathqa} and evaluate on GSM8K \citep{gsm8k} and MATH \citep{math}. 
% Evaluation focuses solely on the final numeric answer. 
We apply LoRA modules to the key, value, query, attention output, and all fully connected weight matrices, training with ranks $r = \{32, 64, 96\}$. 
We present results in Table \ref{tab:arithmetic}. LoRA-SB significantly outperforms LoRA-XS across all settings. LoRA-SB outperforms LoRA-based methods ($r=32$) while using \textbf{40x} fewer trainable parameters for Mistral-7B and \textbf{90x} fewer for Gemma-2 9B at ranks $r=96$ and $r=64$, respectively. 
We present training loss curves comparing LoRA-SB and LoRA-XS in Figure \ref{fig:loss-arithmetic}. Thanks to superior initialization, LoRA-SB starts with a lower initial loss compared to LoRA-XS. Further, due to optimal gradient approximation, LoRA-SB maintains a consistently better loss throughout and converges to a superior final value.

\begin{figure*}[!h]
    \centering
    \subfloat[Mistral-7B]{
        \includegraphics[width=0.48\textwidth]{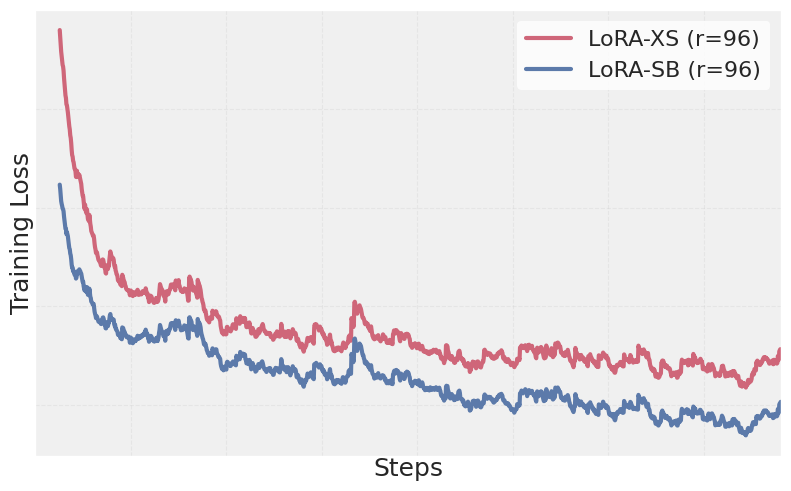}
        \label{fig:loss-mistral-96}
    }
    \hfill
    \subfloat[Gemma-2 9B]{
        \includegraphics[width=0.48\textwidth]{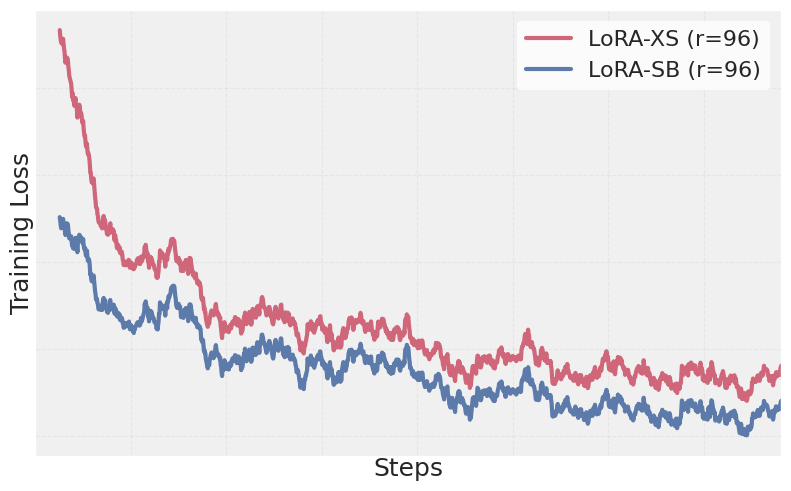}
        \label{fig:loss-gemma-96}
    }
    \caption{Training loss curves for Mistral-7B and Gemma-2 9B, comparing LoRA-SB and LoRA-XS.}
    \label{fig:loss-arithmetic}
\end{figure*}

\subsection{Commonsense Reasoning}
We fine-tune Llama-3.2 3B \citep{llama3} on \textsc{CommonSense170K}, a dataset with eight commonsense reasoning tasks \citep{cr-dataset}. 
% We evaluate the model's performance on each dataset individually.
% , which include BoolQ \citep{clark2019boolq}, SIQA \citep{sap2019socialiqa}, PIQA \citep{bisk2020piqa}, ARC-Challenge \citep{clark2018think}, ARC-Easy \citep{clark2018think}, OBQA \citep{mihaylov2018can}, WinoGrande \citep{sakaguchi2021winogrande}, and HellaSwag \citep{zellers2019hellaswag}.
LoRA modules are applied to the key, value, query, attention output, and all fully connected weight matrices, training with ranks $r = \{32, 64, 96\}$.
We present the results in Table \ref{tab:cr}. LoRA-SB consistently outperforms LoRA-XS across all settings. In addition, LoRA-SB ($r=96$) outperforms LoRA-based methods ($r=32$) with \textbf{27x} fewer trainable parameters.

\subsection{Natural Language Understanding}

We fine-tune RoBERTa-large \citep{liu2019robertarobustlyoptimizedbert} on GLUE, a popular language understanding benchmark. 
% The datasets we evaluate on are: CoLA \citep{warstadt-etal-2019-neural}, RTE, MRPC \citep{dolan2005automatically}, SST-2 \citep{socher2013recursive}, QNLI \citep{rajpurkar2018know}, and STS-B \citep{cer2017semeval}.
% LoRA modules are applied only to the self-attention layers, following the configuration in the original LoRA paper \citep{lora}, with ranks $r = \{8, 16, 24\}$.
LoRA modules are applied only to the self-attention layers, with ranks $r = \{8, 16, 24\}$.
Results are shown in Table \ref{tab:glue}. LoRA-SB consistently outperforms LoRA-XS across all settings. Additionally, LoRA-SB ($r=24$) outperforms LoRA-based methods ($r=8$) with \textbf{39x} lesser trainable parameters.

%% file: analysis.tex
\section{Analysis} \label{sec:analysis}

\textbf{Optimal Initialization is Important!}

To isolate the impact of initialization, we take truncated SVD on various matrices, including Kaiming initialization \cite{kaiming} and $\Delta W_{avg}$ with varying levels of Gaussian noise, as shown in Table \ref{tab:init-matters}. By applying truncated SVD, we ensure optimal gradient approximation, leading to initialization matrices $B_{\text{init}}$ and $A_{\text{init}}$ that form orthonormal bases in $\mathbb{R}^m$ and $\mathbb{R}^n$, respectively. This results in $B^T B = A A^T = I$, allowing us to isolate the effect of initialization. The results clearly demonstrate the significance of initialization, our approach consistently outperforms other variants.

\begin{table}[h]
 \centering
 \caption{
    Comparison of initialization strategies using Mistral-7B on GSM8K and MATH. All methods ensure optimal gradient approximation, with differences arising solely from the initialization. }
 \setlength{\tabcolsep}{8pt}
 \small
 \begin{tabular}{l|cc}
   \toprule
   \multirow{2}{*}{\textbf{Initialization Method}} & \multicolumn{2}{c}{\textbf{Accuracy ($\uparrow$)}} \\
   \cmidrule{2-3}
   & \textbf{GSM8K} & \textbf{MATH} \\
   \midrule
   trunc\_SVD (Kaiming) & $00.00$ & $00.00$ \\
   trunc\_SVD ($\Delta W_{avg} + \mathcal{N}_{\mu=10^{-2}}$) & $00.00$ & $00.00$ \\
   trunc\_SVD ($\Delta W_{avg} + \mathcal{N}_{\mu=10^{-3}}$) & $58.83$ & $14.76$ \\
   trunc\_SVD ($\Delta W_{avg} + \mathcal{N}_{\mu=10^{-4}}$) & $60.19$ & $15.96$ \\
   trunc\_SVD ($\Delta W_{avg} + \mathcal{N}_{\mu=10^{-5}}$) & $60.65$ & $15.98$ \\
   LoRA-SB; trunc\_SVD ($\Delta W_{avg}$) & $\mathbf{63.38}$ & $\mathbf{17.44}$ \\
   \cmidrule{1-3}
 \end{tabular}
 \label{tab:init-matters}
\end{table}

\textbf{Why Do We Use 0.1\% of the Dataset Size for Initialization?}

We selected the $0.1\%$ initialization dataset-size heuristic based on experiments that suggested it provides a good tradeoff between quality and efficiency. 
Specifically, we conducted ablations varying the number of samples used for initialization when fine-tuning Mistral-7B and Gemma-2 9B on 50k samples from MetaMathQA. 
The results (Table \ref{tab:init-samples}) show that once the sample count exceeds a modest threshold (25 samples or $0.05\%$), performance quickly plateaus, indicating that the learned subspace is already sufficiently representative. Using $0.1\%$ of the training data (50 samples) consistently exceeds this threshold across tasks and models, while incurring negligible training time overhead. 
%This demonstrates that the proposed heuristic is simple, robust, and effective across settings, and that searching for an “optimal” subset size is unnecessary.

\begin{table}[h]
 \centering
 \caption{
    Performance effect of number of samples used for initialization.}
 \setlength{\tabcolsep}{8pt}
 \small
 \begin{tabular}{l|cc|cc}
   \toprule
   \multirow{2}{*}{\textbf{\# Samples}} & \multicolumn{2}{c|}{\textbf{Mistral-7B}} & \multicolumn{2}{c}{\textbf{Gemma-2 9B}} \\
   \cmidrule{2-5}
   & \textbf{GSM8K ($\uparrow$)} & \textbf{MATH ($\uparrow$)} & \textbf{GSM8K ($\uparrow$)} & \textbf{MATH ($\uparrow$)} \\
   \midrule
   $1$   & $62.13$  & $15.55$ & $76.03$ & $35.77$ \\
   $5$   & $62.78$ & $16.86$ & $77.49$ & $37.24$ \\
   $25$  & $63.28$ & $17.30$ & $78.18$ & $37.70$ \\
   $50$  & $63.38$ & $17.44$ & $78.40$ & $37.70$ \\
   $100$ & $63.34$ & $17.25$ & $78.22$ & $37.45$ \\
   $200$ & $63.45$ & $17.36$ & $78.43$ & $37.87$ \\
   $500$ & $63.40$ & $17.52$ & $78.54$ & $37.63$ \\
   \cmidrule{1-5}
 \end{tabular}
 \label{tab:init-samples}
\end{table}

\textbf{Optimal Gradient Approximation is Important!}

We aim to examine the effect of optimal gradient approximation. Specifically, we want $B_{\text{init}} R_{\text{init}} A_{\text{init}} \approx \Delta W_{avg}$ without enforcing $B^T B = A A^T = I$. We achieve this through:
\begin{gather} 
    U, S, V^T \leftarrow \textbf{SVD}(\Delta W_{avg}) \label{W_svd_abl} \\ 
    B_{\text{init}} \leftarrow U[1:r] S[1:r, 1:r],
    A_{\text{init}} \leftarrow V[1:r],
    R_{\text{init}} \leftarrow I \label{S_init_abl}
\end{gather}
This ensures that $B_{\text{init}} R_{\text{init}} A_{\text{init}} \approx \Delta W_{avg}$, but only $AA^T = I$, while $B^T B \neq I$. The setup is suboptimal for gradient approximation since we do not explicity use the closed-form solution derived in Theorem \ref{theorem:pro}.
We compare the resulting loss curves against LoRA-SB (which uses optimal gradient approximation) for Mistral-7B, as shown in Figure \ref{fig:loss-ablation} in Appendix \ref{app:init}. Although both start similarly due to effective initialization, LoRA-SB converges to significantly better values, demonstrating the advantage of optimal gradient approximation. Furthermore, LoRA-SB achieves higher accuracies on GSM8K and MATH, with scores of $63.38$ and $17.44$ compared to $55.87$ and $12.74$, respectively.

\textbf{Training Time and Inference.} 

We provide detailed benchmarks of training time and inference performance in Appendix \ref{app:training-time} and \ref{app:inference}, respectively. 
As shown, the initialization step in LoRA-SB introduces only a negligible training-time overhead compared to LoRA ($\approx 1.1\%-1.3\%$).

%% file: conclusion.tex
\section{Conclusion}

In this work, we introduced LoRA-SB, which bridges the gap between low-rank PEFT and full FT. 
% Unlike traditional methods, which rely on access to high-rank gradients or updates, LoRA-SB achieves comparable performance entirely within low-rank subspaces. 
This is enabled by our initialization strategy, which approximates the first step of full FT and ensures that the most relevant subspaces for task-specific adaptation are captured. 
% Our theoretical analysis uncovered critical limitations in LoRA-XS: suboptimal gradient approximation and reliance on fixed initialization that fails to adapt to task-specific nuances. 
%By addressing both issues through a unified approach, 
We achieve optimal gradient scaling and preserve update directions throughout training. 
Our approach ensures scaling factor independence by approximating the full FT gradient, thereby eliminating potential instability issues.
Through extensive experiments, we demonstrate that our method outperforms LoRA (and baselines) using upto \textbf{90x} less parameters, and comprehensively outperforms LoRA-XS. 
%Our work advances PEFT while laying the groundwork for further innovations in low-rank adaptation.

% Through extensive experiments on 4 models across 16 datasets covering mathematical reasoning, commonsense reasoning, and language understanding tasks, we demonstrate that our method exceeds the performance of LoRA while using upto \textbf{90x} less parameters, and comprehensively outperforms LoRA-XS. 

%We also believe the property of fixing subspaces can be useful and looked as a different way of FT models. 

%% file: ack.tex
\section{Acknowledgements}
This research was supported by funding from Mohamed bin Zayed University of Artificial Intelligence (MBZUAI) and ADIA Lab.

%% file: appendix.tex
% print a mini‑ToC for everything from here on
% \section*{Appendix Contents}
% % if you want “Appendix Contents” in the main ToC, uncomment next line:
% % \addcontentsline{toc}{section}{Appendix Contents}
% \etocsettocdepth{section}    % or subsection, subsubsection…
% \localtableofcontents
% -------------------------------------------------

\section{Related Work} \label{app:related-work}

\textbf{Parameter-Efficient Fine-Tuning (PEFT).}
PEFT methods have become essential for adapting large pre-trained models under computational constraints. Early techniques like AdapterFusion \cite{Pfeiffer_Kamath_Rücklé_Cho_Gurevych_2021} and Prefix-Tuning \cite{Li_Liang_2021} enabled task-specific adaptation with minimal parameter updates. Advances like soft prompts \cite{Lester_Al-Rfou_Constant_2021} further reduced trainable parameter counts while maintaining strong performance. Recent approaches have explored operating directly on model representations \cite{Wu_Arora_Wang_Geiger_Jurafsky_Manning_Potts_2024}.

\textbf{Low-Rank Decomposition Methods.}
LoRA \cite{lora} demonstrated that weight updates during FT could be efficiently approximated using low-rank matrices, drastically reducing parameter counts. Building on this insight, variants such as QLoRA \cite{Dettmers_Pagnoni_Holtzman_Zettlemoyer_2023} and AdaLoRA \cite{Zhang_Chen_Bukharin_Karampatziakis_He_Cheng_Chen_Zhao_2023} extended the paradigm through quantization and adaptive allocation strategies. The applicability of low-rank techniques has also been explored in pretraining with GaLore \cite{Zhao_Zhang_Chen_Wang_Anandkumar_Tian_2024} and ReLoRA \cite{Lialin_Shivagunde_Muckatira_Rumshisky_2023}, highlighting the versatility of low-rank adaptation methods. 
LoRA-based methods have also been applied in other domains, such as efficient federated FT \citep{sun2024improvingloraprivacypreservingfederated, singhal2024exact}.

\textbf{Enhancing LoRA Performance.}
Recent efforts have focused on optimizing LoRA's performance. PiSSA \cite{pissa} demonstrated improvements by initializing matrices with principal components of pre-trained weights. LoRA-Pro \cite{LoRA-Pro} and LoRA-GA \cite{lora_ga} improved gradient approximation, aligning low-rank updates more closely with full FT. Methods like DoRA \cite{Liu_Wang_Yin_Molchanov_Wang_Cheng_Chen_2024} and rsLoRA \cite{rslora} introduced decomposition-based and scaling stabilization techniques to enhance learning stability and expand LoRA's utility.

\textbf{Improving Efficiency in LoRA Variants.}
Efficiency-focused innovations have pushed LoRA toward more parameter savings. LoRA-XS \citep{LoRA-XS} achieves this by inserting a small trainable weight matrix into frozen low-rank matrices. VeRA \cite{Kopiczko_Blankevoort_Asano_2024} shares low-rank matrices across layers, relying on scaling vectors for task-specific adaptation. Tied-LoRA \cite{Renduchintala_Konuk_Kuchaiev_2024} leverages weight tying to reduce parameter usage at higher ranks, while HydraLoRA \cite{Tian_Shi_Guo_Li_Xu_2024} introduces an asymmetric architecture for improvement.

\section{Proofs}\label{app:proofs}
In all the proofs below, we will use the notations defined in Section \ref{sec:Methodologies}.
\subsection{Proof of Lemma \ref{lemma:subspace}} \label{app:proof_lemma_subspace}

\begin{tcolorbox}[colback=cyan!10,colframe=black]
\begin{lemma*}\label{lemma:Subspace}
Let $\Delta W$ be an update learned with LoRA-XS. Then, the set of all possible $\Delta W$, say $ \mathcal{  W}_{LoRA-XS}$, is given as:
$$
\mathcal{  W}_{LoRA-XS} =\{M \in \mathbb{R}^{m \times n}|
\text{Col}(M) \subseteq \text{Col}(B) \wedge \text{Row}(M) \subseteq  \text{Row}(A)\},
$$
where $\text{Col}(M)$ and $\text{Row}(M)$ are column and row spaces of matrix $M$ respectively.
\end{lemma*}
\end{tcolorbox}

\begin{proof}
Since $\Delta W = BRA$, we have
\begin{align*}
    \text{Col}(\Delta W) = \{y \in \mathbb{R}^m \mid y = BRAx, \, x \in \mathbb{R}^n\} \implies \\
    \text{Col}(\Delta W) = \{y \in \mathbb{R}^m \mid y = Bz, \, z \in \text{Col}(RA)\} \subseteq \text{Col}(B).
\end{align*}

That is, we proved that
\begin{align}
    \text{Col}(\Delta W) &\subseteq \text{Col}(B).
\end{align}
% A symmetric proof can now be given for row space of $\Delta W$.
Following similar arguments, one can also show $\text{Row}(\Delta W) \subseteq \text{Row}(A)$.
\end{proof}

\subsection{Proof of Lemma \ref{lemma:gR}}\label{app:proof_lemma_gr}

\begin{tcolorbox}[colback=cyan!10,colframe=black]
\begin{lemma*}
The gradient of the loss with respect to matrix $R$ can be expressed in terms of the gradient with respect to the weight matrix $W$ as:
\begin{gather*}
g^R_{LoRA-XS} = s B^\top  g A^\top.
\end{gather*}
\end{lemma*}
\end{tcolorbox}

\begin{proof}
    Let $L$ be the loss function. We have already defined $g$ and $g^R_{\text{LoRA-XS}}$ as:
    \begin{align}
        g := \frac{\partial L}{\partial W}  \quad \& \quad g^R_{\text{LoRA-XS}}:= \frac{\partial L}{\partial R}. 
    \end{align}
    The chain rule gives
    \begin{align}
        \frac{\partial L}{\partial R} = \frac{\partial L}{\partial W} \frac{\partial W}{\partial R} \implies
        \frac{\partial L}{\partial R} = \frac{\partial L}{\partial W} \frac{\partial W}{\partial X} \frac{\partial X}{\partial R} \quad \text{ for } X=RA
    \end{align}
    
    We know that for $W = sBX$:
    \begin{align}
        \frac{\partial L}{\partial W}\frac{\partial W}{\partial X} = sB^\top g 
        \implies \frac{\partial L}{\partial R} = sB^\top g\frac{\partial X}{\partial R}
    \end{align}
    
    Let $sB^\top g = y$. We know that when $X = RA$:
    \begin{align}
        y\frac{\partial X}{\partial R} = yA^\top 
        \implies \frac{\partial L}{\partial R} = yA^\top  = sB^\top gA^\top  \
    \end{align}
\begin{align}
       \text{Therefore, \quad}  \boxed{g^R_{\text{LoRA-XS}} = sB^\top gA^\top }
\end{align}
\end{proof}

\subsection{Proof of Theorem \ref{theorem:pro}}\label{app:proof_pro}

\begin{tcolorbox}[colback=cyan!10,colframe=black]
\begin{theorem*}
For full-rank $A$ and $B$ matrices, the optimal solution for the objective  \\$\text{min}_{g^R} ||\tilde{g} - g||^2_F$, such that $\Tilde{g} = sB g^R A$, is:
$
g^R = \dfrac{1}{s^2} (B^\top  B)^{-1} g^R_{LoRA-XS} (A A^\top )^{-1}
$.
\end{theorem*}
\end{tcolorbox}

\begin{proof}
Since we already defined the equivalent gradient $\Tilde{g} := sB g^R A$, the minimization problem can be denoted as:
\begin{align}
    \argmin_{g^R} F &= \|sB g^R A - g\|_F^2 
\end{align}

For differentiable $F$,
\begin{align}
    \frac{\partial F}{\partial g^R} &= 0 \implies
    2(\tilde{g} - g) \cdot \frac{\partial \tilde{g}}{\partial g^R} = 0 \implies
    2(sBg^RA - g) \cdot \frac{\partial (sBg^RA)}{\partial g^R} = 0
\end{align}

Using the same trick from before and substituting $g^RA = X$, we get:
\begin{align}
    2sB^\top (sBg^RA - g)A^\top  = 0 \implies
     B^\top (sBg^RA - g)A^\top  = 0 \implies
    B^\top sBg^RAA^\top  = B^\top gA^\top 
\end{align}

From Lemma \ref{lemma:gR}, we get:
\begin{align}
    B^\top gA^\top  = g^R_{\text{LoRA-XS}}/s \implies
   B^\top sBg^RAA^\top  = g^R_{\text{LoRA-XS}}/s \implies
    B^\top Bg^RAA^\top  = g^R_{\text{LoRA-XS}}/s^2
\end{align}

Now since $B$ and $A$ are full rank, multiplying both sides by $(B^\top B)^{-1}$ and $(AA^\top )^{-1}$ on the left and right side respectively gives:
\begin{align}
    (B^\top B)^{-1}(B^\top Bg^RAA^\top )(AA^\top )^{-1} &= (B^\top B)^{-1}g^R_{\text{LoRA-XS}}(AA^\top )^{-1}/s^2
\end{align}

\begin{equation}
    \text{Therefore,} \quad \boxed{g^R = \dfrac{1}{s^2} (B^\top  B)^{-1} g^R_{\text{LoRA-XS}} (A A^\top )^{-1}}
\end{equation}
\end{proof}

\subsection{Proof of Theorem \ref{theorem:loss-neg}}
\label{app:proof_loss-neg}

\begin{tcolorbox}[colback=cyan!10,colframe=black]
\begin{theorem*}
Consider the update for matrix $R$ using the solution derived in Theorem \ref{theorem:pro}:
\begin{gather*}
R\leftarrow R - \eta g^R
\end{gather*}
where $\eta > 0$ is the (sufficiently small) learning rate.
This update guarantees a reduction in the loss $\Delta L$, given by:
\begin{gather*}
\Delta L \coloneqq L(W_0 + sB(R-\eta g^R)A) - L(W_0 + sBRA) 
% = \langle \frac{\partial L}{\partial R}, \dd R \rangle_F 
= -\eta \langle g^R_{LoRA-XS}, g^R \rangle_F + o(\eta)
\le 0.
\end{gather*}
\end{theorem*}
\end{tcolorbox}

\begin{proof}
% We establish that during optimization, the differential change in the loss function, $\dd L$, can be expressed as:
Assuming that $L$ is differentiable, we use Taylor's theorem and get
\begin{align}
    \Delta L &\coloneqq L(W_0 + sB(R-\eta g^R)A) - L(W_0 + sBRA)  \notag\\
    &= \left\langle \frac{\partial L}{\partial R}, -\eta g^R \right\rangle_F + o(\eta) \notag\\
    &= -\frac{\eta}{s^2} \langle g^R_{\text{LoRA-XS}}, (B^\top B)^{-1}g^R_{\text{LoRA-XS}}(AA^\top )^{-1} \rangle_F + o(\eta), \label{eq:loss_differential}
\end{align}

where in the last step we also used the definition of $g^R_{\text{LoRA-XS}}$ and the result of Theorem~\ref{theorem:pro}. To prove $\Delta L \leq 0$ for small enough $\eta$, it is sufficient to show that

\begin{equation}
    \langle g^R_{\text{LoRA-XS}}, (B^\top B)^{-1}g^R_{\text{LoRA-XS}}(AA^\top )^{-1} \rangle_F \geq 0.
\end{equation}
Next, we note that matrices $B^\top B \in \mathbb{R}^{r\times r}$ and $A A^\top \in \mathbb{R}^{r\times r}$ are positive definite since they are positive semi-definite and matrices $B$ and $A$ are full-rank (i.e., with rank $r$) matrices, which means that $B^\top B$ and $AA^\top$ have non-zero eigenvalues. Therefore, $(B^\top B)^{-1}$ and $(AA^\top)^{-1}$ are also positive definite, implying that there exist matrices $X$ and $Y$ such that $(B^\top B)^{-1} = YY^\top$ and $(AA^\top)^{-1} = XX^\top$ (e.g., one can find such matrices using Cholesky decomposition). Then, we have

\begin{equation*}
\begin{aligned}
    \langle g^R_{\text{LoRA-XS}}, (B^\top B)^{-1}g^R_{\text{LoRA-XS}}(AA^\top )^{-1} \rangle_F
    & = \langle g^R_{\text{LoRA-XS}}, YY^\top g^R_{\text{LoRA-XS}}XX^\top  \rangle_F \\
    & = \langle Y^\top g^R_{\text{LoRA-XS}}X, Y^\top g^R_{\text{LoRA-XS}}X \rangle_F \\
    & = \|Y^\top g^R_{\text{LoRA-XS}}X\|_F^2 \geq 0.
\end{aligned}
\end{equation*}
This concludes the proof.
% Thus, we have proven:
% \begin{equation}
%     \dd L = -\eta \langle g^R_{\text{LoRA-XS}}, \frac{1}{s^2}(B^\top B)^{-1}g^R_{\text{LoRA-XS}}(AA^\top )^{-1} \rangle_F  \leq 0
% \end{equation}
\end{proof}

For our specific initialization where $(B^\top B) = I$, $(AA^\top ) = I$, and $s=1$, the result simplifies to:
\begin{equation}
\begin{aligned}
    \Delta L & = -\eta \langle g^R_{\text{LoRA-XS}}, g^R_{\text{LoRA-XS}} \rangle_F  + o(\eta) \leq 0.
\end{aligned}
\end{equation}

\subsection{Proof of Theorem \ref{theorem:hyper}} \label{app:proof_hyper}
\begin{tcolorbox}[colback=cyan!10,colframe=black]
\begin{theorem*}
The equivalent gradient $\Tilde{g}$ is  hyperparameter $s$ independent when 
\begin{gather*}
    \Tilde{g}=sBg^RA \quad
\text{ but not when }\quad 
    \Tilde{g}=sBg^R_{LoRA-XS}A.
\end{gather*}
\end{theorem*}
\end{tcolorbox}
\begin{proof}

Let $g$ be the full fine-tuning gradient. We want to prove that $\tilde{g}$ does not depend on $s$, so we try to express it in terms of $g$ which does not depend on the \xs training process or reparameterization. 

1) For $\tilde{g} = sBg^RA$:
\begin{align}
    g^R = \frac{1}{s^2} (B^\top B)^{-1} g^R_{\text{LoRA-XS}} (AA^\top )^{-1} 
    \implies \tilde{g} = \frac{s}{s^2} B(B^\top B^{-1})g^R_{\text{LoRA-XS}}(AA^\top )^{-1}A
\end{align}

Now since $g^R_{\text{LoRA-XS}} = sB^\top gA^\top $:
\begin{align}
    \tilde{g} = \frac{1}{s} B(B^\top B^{-1})sB^\top gA^\top (AA^\top )^{-1}A 
    = B(B^\top B^{-1})B^\top gA^\top (AA^\top )^{-1}A.
\end{align}
which is $s$-independent.

2) For $\tilde{g} = sB g^R_{\text{LoRA-XS}}A$\\
\begin{align}   
  g^R_{\text{LoRA-XS}} = sB^\top gA^\top  \implies \tilde{g} = sB(sB^\top gA^\top )A 
    \implies \tilde{g} = s^2BB^\top gA^\top A
\end{align}
which is not $s$-independent.
\end{proof}

\subsection{Proof of Theorem \ref{theorem:grad_init}}\label{app:proof_grad_init}
\begin{tcolorbox}[colback=cyan!10,colframe=black]
\begin{theorem*} 
If $A_{init}$ and $B_{init}$ are initialized using \ours for the first step of SGD optimizer, then the update given by LoRA-SB, $\Delta(B_{init} R_{init}A_{init})$ , is the best low-rank approximation of full fine-tuning update,
$
 \Delta W.
$
\end{theorem*}
\end{tcolorbox}
\begin{proof}
Consider a gradient descent step with learning rate $\eta$ and updates for $R$:
\begin{align}
    \Delta R = -\eta \nabla_R \mathcal{L}(R) \implies
    B\Delta RA = -\eta B \nabla_R \mathcal{L}(R)A.
\end{align}
To measure its approximation quality of update of the weights in full finetuning:
\begin{align}
    \Delta W &= -\eta \nabla_W \mathcal{L}(W_0).
\end{align}
We use Frobenius norm of the difference between these two updates as a criterion:
\begin{align}
    \|B\Delta R A - \eta \nabla \mathcal{L}_W(W_0)\|_F
    &= \eta \|B\nabla_R \mathcal{L}(R)A - \nabla \mathcal{L}_W(W_0)\|_F.
\end{align}
We have shown before that:
\begin{align}
    \nabla_R \mathcal{L} &= B^\top  \nabla_W\mathcal{L} A^\top.
\end{align}
The problem now becomes:
\begin{align}
    &\min_{A_{\text{init}}, B_{\text{init}}} \|B^\top (B^\top  \nabla_W\mathcal{L} A^\top )A - \nabla_W\mathcal{L}\|_F 
    \quad \text{ where } \nabla_W\mathcal{L} = USV^\top. 
\end{align}
Using our initialization, we get:
\begin{align}
    \|BB^\top \nabla_W\mathcal{L}A^\top A - \nabla_W\mathcal{L}\|_F
    &= \|U_{IR}U_{IR}^\top USV^\top V_{IR}V_{IR}^\top  - USV^\top \|_F.
\end{align}
Moreover, we also have
\begin{align}
    U_{IR}U_{IR}^\top USV^\top V_{IR}V_{IR}^\top  &= \sum_{i=1}^r \sigma_i u_iv_i^\top.
\end{align}
The rank of $W'$ such that
\begin{align}
    W' &= U_{IR}U_{IR}^\top USV^\top V_{IR}V_{IR}^\top 
\end{align}
is $\leq r$, since the corresponding ranks of $B_{\text{init}}$ and $A_{\text{init}}$ is $r$.
Using the Eckart-Young Theorem, we find the optimal low-rank solution as:
\begin{align}
    W'^* &= \underset{\text{rank}(W')=r}{\text{arg min}} \|W' - \nabla_W\mathcal{L}\|_F = \sum_{i=1}^r \sigma_i u_iv_i^\top.
\end{align}
Since we also get an identical expression, our solution is optimal.
\end{proof}

\section{Simulating the First Step of Full Fine-Tuning Under AdamW}
\label{app:adam_sign}

Our initialization is designed to approximate the first update step that would occur during full fine-tuning using the AdamW optimizer, which is also used in LoRA-SB training. AdamW computes the parameter update using both first and second moment estimates of the gradient. At the first step, these moments are initialized to zero, so the update becomes:
\[
\theta_1 = \theta_0 - \alpha \cdot \frac{g_1}{\sqrt{g_1^2 + \epsilon}} \approx -\alpha \cdot \text{sign}(g_1)
\]
where \( g_1 \) is the gradient at the first step, \( \epsilon \) is a small constant for numerical stability, and \( \alpha \) is the learning rate. Due to zero-initialization and bias correction, the direction of the update is approximately the element-wise sign of the gradient.

To simulate this behavior in our low-rank initialization, we use:
\[
\Delta W_{\text{avg}} = -\eta \cdot \text{sign}\left( \sum_{i=1}^n \nabla_W \mathcal{L}(W_0, x_i) \right)
\]
This reflects the direction of the first AdamW step averaged over a mini-batch. By using the sign of the gradient sum, we ensure our initialization aligns with the dynamics of AdamW, leading to a consistent and faithful approximation of full fine-tuning updates within the low-rank subspace.

\section{Algorithm} \label{app:algo}

We provide a pseudo-code implementation of our method in Algorithm \ref{alg:lora-sb}.

\begin{center}
\begin{minipage}{0.6\textwidth}
  \begin{algorithm}[H]
    \caption{\textcolor{MidnightBlue}{\texttt{\textbf{ LoRA‑SB, PyTorch‑like}}}}
    \label{alg:lora-sb}
    \begin{algorithmic}[1]
      \STATE \texttt{\textcolor{BrickRed}{\textbf{def}}}%
             \ \texttt{\textcolor{MidnightBlue}{\textbf{initSB}}\textbf{(model, D)}}
      \STATE \hspace{1.5em}\texttt{\textcolor{OliveGreen}{\# \textbf{Estimate gradient with n samples}}}
      \STATE \hspace{1.5em}\texttt{$\Delta W_{\mathrm{avg}} \leftarrow$%
             \ \textcolor{Plum}{\textbf{est\_grad}}(model, D, n)}
      \STATE \hspace{1.5em}\texttt{\textcolor{OliveGreen}{\# \textbf{Initialize B, R, A}}}
      \STATE \hspace{1.5em}\texttt{$(B,R,A) \leftarrow$%
             \ \textcolor{Plum}{\textbf{trunc\_SVD}}($\Delta W_{\mathrm{avg}}$)}
      \STATE \hspace{1.5em}\texttt{\textcolor{OliveGreen}{\# \textbf{Convert to LoRA‑SB model}}}
      \STATE \hspace{1.5em}\texttt{sb\_model} $\leftarrow$%
        \ \textcolor{Plum}{\textbf{\texttt{lora\_SB}}(model, B, R, A)}
      \STATE \hspace{1.5em} \texttt{\textbf{return}\,sb\_model}

      \STATE
      \STATE \texttt{\textcolor{OliveGreen}{\# \textbf{Load pre‑trained model}}}
      \STATE\texttt{ model $\leftarrow$ \textbf{\textcolor{Plum}{AutoModel}}(base\_model)}
      \STATE \texttt{\textcolor{OliveGreen}{\# \textbf{Initialize LoRA‑SB with D}}}
      \STATE \texttt{sb\_model $\leftarrow$ \textbf{\textcolor{MidnightBlue}{initSB}}(model, D)}
      \STATE \texttt{\textcolor{OliveGreen}{\# \textbf{Train, only R trainable}}}
      \STATE \texttt{trainer $\leftarrow$ \textbf{\textcolor{Plum}{Trainer}}(sb\_model,…)}
      \STATE \texttt{trainer.train()}
    \end{algorithmic}
  \end{algorithm}
\end{minipage}
\end{center}

\section{Optimal Gradient Approximation is Important!} \label{app:init}
As discussed in Section \ref{sec:analysis}, optimal gradient approximation plays a key role in the effectiveness of LoRA-SB. In Figure \ref{fig:loss-ablation}, we compare the loss curves of models trained with and without this component on Mistral-7B. While both variants begin with similar performance due to effective initialization, LoRA-SB with optimal gradient approximation converges to substantially lower loss values, highlighting its contribution to improved optimization.

\begin{figure}[h]
    \centering
    \includegraphics[width=0.6\linewidth]{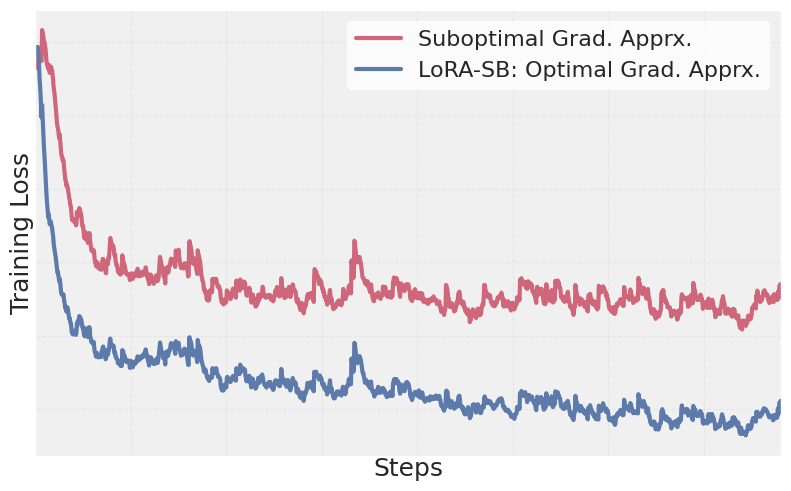}
    \caption{Training loss for Mistral-7B, highlighting the impact of optimal gradient approximation.}
    \label{fig:loss-ablation}
\end{figure}

\section{Training Time Overhead vs LoRA-XS} \label{app:training-time}

As previously mentioned, we compute the update approximation using only $1/1000$ of the total training samples for each dataset. Table \ref{tab:training_time} presents the associated training time overhead for these computations, compared to LoRA-XS. The results show that the \textbf{additional overhead is negligible}, adding just $2–4$ minutes compared to the total training time of $3–5$ hours per epoch ($\approx 1.1\%$ to $1.3\%$). Additionally, the update computation is performed only once, at the beginning of the first epoch, prior to training.
Notably, the initialization step is highly efficient, as we directly compute the \textbf{truncated SVD} using optimized PyTorch libraries (\textbf{\texttt{torch.svd\_lowrank}}). For reference, this \textbf{computation takes less than one second for each of the entire LLMs} used in our experiments.

\begin{table}[h]
\centering
\caption{
  Training time overhead due to the initialization for various models on their respective tasks.
}
\setlength{\tabcolsep}{6pt}
\small
\begin{tabular}{l|cc}
  \toprule
  \textbf{Model} & \textbf{Overhead} & \textbf{Training Time/Epoch} \\
  \midrule
  Mistral-7B & 0:02:01 & 3:03:57 \\
  Gemma-2 9B & 0:03:46 & 4:13:24 \\
  Llama-3.2 3B & 0:03:54 & 4:54:31 \\
  \bottomrule
\end{tabular}
\label{tab:training_time}
\end{table}

\section{Inference Overhead vs LoRA} \label{app:inference}

LoRA-SB introduces a minimal inference cost overhead due to the insertion of the $r \times r$ matrix $R$ between $B$ and $A$, and the need for higher ranks to achieve comparable performance to LoRA. We benchmark the inference-time FLOPs and MACs across various models and find that the overhead is negligible.
This comparison is presented in Table \ref{tab:inference}, showing that the additional overhead of LoRA-SB is negligible.

\begin{table}[!htbp]
\centering
\caption{Inference cost comparison between LoRA-SB and LoRA across various models for a sequence length of 256. The minimum rank at which LoRA-SB matches or exceeds LoRA's performance is highlighted in \textbf{bold}.}
\setlength{\tabcolsep}{6pt}
\small
\begin{tabular}{lcc|cc}
\toprule
\textbf{Model} & \textbf{Method} & \textbf{Rank} & \textbf{MACs} & \textbf{FLOPs} \\
\midrule
\multirow{3}{*}{RoBERTa-large} 
& LoRA & $8$ & $77.86$ G & $155.79$ G \\
& LoRA-SB & $16$ & $78.42$ G & $156.91$ G \\
& \textbf{LoRA-SB} & $24$ & $78.97$ G & $158.01$ G \\
\midrule
\multirow{3}{*}{LlaMA-3.2 3B}
& LoRA & $32$ & $0.84$ T & $1.67$ T \\
& LoRA-SB & $64$ & $0.85$ T & $1.70$ T \\
& \textbf{LoRA-SB} & $96$ & $0.86$ T & $1.72$ T \\
\midrule
% \multirow{3}{*}{Qwen-2.5 7B}
% & LoRA & $32$ & $1.83$ T &  $3.66$ T\\
% & LoRA-SB & $64$ & $1.84$ T & $3.70$ T \\
% & LoRA-SB & $92$ & $1.85$ T & $3.75$ T \\
% \midrule
\multirow{3}{*}{Mistral 7B}
& LoRA & $32$ & $1.84$ T & $3.69$ T \\
& LoRA-SB & $64$ & $1.86$ T & $3.73$ T \\
& \textbf{LoRA-SB} & $92$ & $1.88$ T & $3.77$ T\\
\midrule
\multirow{3}{*}{Gemma-2 9B}
& LoRA & $32$ & $3.89$ T & $7.77$ T  \\
& \textbf{LoRA-SB} & $64$ & $3.93$ T & $7.86$ T \\
& LoRA-SB & $96$ & $3.97$ T & $7.94$ T \\
\bottomrule
\end{tabular}
\label{tab:inference}
\end{table}

\section{Experiment Details} \label{app:exps}

We use PyTorch \citep{paszke2019pytorch} and the HuggingFace Transformers library \citep{wolf2020transformers} for our implementations. We run all experiments on a \textbf{single NVIDIA A6000 GPU} and report results as the average of three random seeds. To save memory, we initialize base models in \texttt{\textbf{torch.bfloat16}} precision. We trained all models using the AdamW optimizer \citep{loshchilov2019decoupledweightdecayregularization}. \textbf{We compute the update approximation using only $\mathbf{1/1000}$ of each dataset's total number of samples}. The samples are randomly selected from the training set in each run. 

For arithmetic and commonsense reasoning tasks, we set up Mistral-7B, Gemma-2 9B, and Llama-3.2 3B with hyperparameters and configurations listed in Table \ref{tab:hyper_it}. We adopted most settings from previous studies \citep{cr-dataset} but conducted our own learning rate sweep. Following LoRA-XS guidelines, we set $\alpha = r$ for their baseline configuration.

For the GLUE benchmark using RoBERTa-large, you can find the hyperparameter details in Table \ref{tab:hyper_roberta}. We mostly adhered to the original configurations from the LoRA paper \citep{lora} but adjusted the learning rate through a sweep. In line with LoRA-XS settings, we fixed $\alpha$ at $16$ for their baseline. 

For all tasks, we followed the baseline configurations provided in the PiSSA \cite{pissa}, rsLoRA \cite{rslora}, DoRA \cite{Liu_Wang_Yin_Molchanov_Wang_Cheng_Chen_2024}, and LoRA-Pro \cite{LoRA-Pro} papers for our comparisons.

\begin{table}[ht]
\centering
\caption{
Hyperparameter settings for training Mistral-7B and Gemma-2 9B on MetaMathQA, and Llama-3.2 3B on \textsc{Commonsense170K}.}
\begin{tabular}{l|cc}
\hline
\toprule
 & \textbf{Mistral-7B / Gemma-2 9B} & \textbf{Llama-3.2 3B}\\
\midrule

Optimizer & \text{AdamW} & \text{AdamW} \\
Batch size & $1$ & $6$ \\
Max. Seq. Len & $512$ & $256$ \\
Grad Acc. Steps & $32$ & $24$ \\
Epochs & $1$ & $2$ \\
Dropout & $0$ & $0.05$\\
Learning Rate & $1\times10^{-4}$ & $2\times10^{-3}$\\
LR Scheduler & Cosine & Linear\\
Warmup Ratio & $0.02$ & $0.02$\\
\bottomrule
\end{tabular}

\label{tab:hyper_it}
\end{table}

\begin{table}[h]
   %\footnotesize
   %\addtolength{\tabcolsep}{-1pt}
   \centering
      \caption{Hyperparameter settings for RoBERTa-large on GLUE.}
   \begin{tabular}{l|cccccc}
    \hline
    \toprule
    & \textbf{CoLA} & \textbf{RTE} & \textbf{MRPC} & \textbf{SST-2} & \textbf{QNLI} & \textbf{STS-B} \\
    \midrule
    Optimizer & \multicolumn{6}{c}{AdamW} \\
    Batch size & \multicolumn{6}{c}{128}\\
    Max Seq. Len. & \multicolumn{6}{c}{$256$} \\
    Epochs & $30$ & $30$ & $30$ & $15$ & $15$ & $30$\\
    Dropout & \multicolumn{6}{c}{$0$}\\
    Learning Rate  & \multicolumn{6}{c}{$1\times10^{-3}$}\\
    LR Scheduler & \multicolumn{6}{c}{Linear} \\
    Warmup Ratio   & \multicolumn{6}{c}{$0.06$} \\
    \bottomrule
   \end{tabular}

   \label{tab:hyper_roberta}
\end{table}

\section{Dataset Details} \label{app:datasets}

The \textbf{MetaMathQA} dataset \citep{metamathqa} creates mathematical questions by rephrasing existing ones from different viewpoints, without adding new information. We assess this dataset using two benchmarks: \textbf{GSM8K} \citep{gsm8k}, which consists of grade-school math problems requiring multi-step reasoning, and \textbf{MATH} \citep{math}, which presents difficult, competition-level math problems.
Evaluation focuses solely on the final numeric answer.

\textsc{\textbf{CommonSense170K}} is a comprehensive dataset that consolidates eight commonsense reasoning datasets \citep{cr-dataset}.  Each example is framed as a multiple-choice question where the model generates the correct answer without explanations. We use the prompt template from \citep{cr-dataset}. The individual datasets used are described below:

\begin{enumerate} 
    \item \textbf{HellaSwag} \citep{zellers2019hellaswag} challenges models to select the most plausible continuation of a given scenario from multiple possible endings. 
    \item \textbf{ARC Easy} (or \textbf{ARC-e}) \citep{clark2018think} includes basic science questions at a grade-school level, offering simpler tasks to assess fundamental reasoning abilities. 
    \item \textbf{PIQA} \citep{bisk2020piqa} evaluates physical commonsense reasoning, where models must choose the best action to take in a hypothetical scenario. 
    \item \textbf{SIQA} \citep{sap2019socialiqa} tests social commonsense reasoning by asking models to predict the social consequences of human actions. 
    \item \textbf{WinoGrande} \citep{sakaguchi2021winogrande} presents sentence completion tasks requiring commonsense reasoning to select the correct binary option. 
    \item \textbf{ARC Challenge} (or \textbf{ARC-c}) \citep{clark2018think} consists of more complex science questions designed to challenge models with sophisticated reasoning, beyond simple co-occurrence patterns. 
    \item \textbf{OBQA} \citep{mihaylov2018can} features open-book, knowledge-intensive QA tasks that require multi-hop reasoning across multiple information sources. 
    \item \textbf{BoolQ} \citep{clark2019boolq} involves answering yes/no questions based on real-world, naturally occurring queries. 
\end{enumerate}

The \textbf{GLUE Benchmark} is a comprehensive collection of tasks designed to evaluate natural language understanding (NLU) abilities. It included various datasets, including \textbf{STS-B} for measuring semantic textual similarity \citep{cer2017semeval}, \textbf{RTE} for recognizing textual entailment, \textbf{MRPC} for detecting paraphrases \citep{dolan2005automatically}, \textbf{CoLA} for assessing linguistic acceptability \citep{warstadt-etal-2019-neural}, \textbf{SST-2} for sentiment analysis \citep{socher2013recursive}, and \textbf{QNLI} for question-answer inference \citep{rajpurkar2018know}. GLUE’s broad scope makes it a standard benchmark for evaluating models like RoBERTa.

% \section{Limitations} \label{section:limits}
% We have not yet integrated adaptive layer-wise rank selection or explored combining LoRA-SB with quantization, both promising directions for future work. 
% Additionally, our evaluation is limited to language models; applying LoRA-SB to other architectures like Vision-Language Models (VLMs) and Vision Transformers (ViTs) remains to be explored. 
% Finally, we do not address adapter fusion, which involves merging adapters trained on different tasks, though the frozen $A$ and $B$ matrices may allow for constructing orthogonal subspaces to enable conflict-free merging.

\section{Use of Large Language Models} \label{app:llm_use}

LLMs are only used for small writing improvements, like polishing grammar and smoothing out phrasing.